\documentclass{article} 
\usepackage{iclr2026_conference,times}

\usepackage[utf8]{inputenc} 
\usepackage[T1]{fontenc}    
\usepackage[dvipsnames]{xcolor}         
\definecolor{Grey}{rgb}{0.5,0.5,0.5}
\definecolor{InterBlue}{RGB}{1, 30, 160} 

\usepackage[colorlinks=True, citecolor=InterBlue, linkcolor=InterBlue, urlcolor=InterBlue]{hyperref}
\usepackage{url}            
\usepackage{booktabs}       
\usepackage{amsfonts}       
\usepackage{nicefrac}       
\usepackage{microtype}      
\usepackage{amsthm}        

\usepackage{amsmath,amsfonts,bm}

\newcommand{\attn}{\mathrm{attn}}

\def\eqref#1{equation~\ref{#1}}

\def\1{\bm{1}}

\DeclareMathAlphabet{\mathsfit}{\encodingdefault}{\sfdefault}{m}{sl}
\SetMathAlphabet{\mathsfit}{bold}{\encodingdefault}{\sfdefault}{bx}{n}

\newcommand{\softmax}{\mathrm{softmax}}

\usepackage{booktabs,tabularx}
\usepackage{wrapfig}
\newtheorem{definition}{Definition}

\newtheorem{theorem}{Theorem}
\newtheorem{lemma}{Lemma}
\newtheorem{proposition}{Proposition}
\newtheorem{corollary}{Corollary}

\newtheorem{remark}{Remark}

\definecolor{skyblueA}{HTML}{EAF6FF}   
\definecolor{skyblueB}{HTML}{CFEAFF}   
\definecolor{skyblueC}{HTML}{4FA3E3}   
\definecolor{skyblueD}{HTML}{1D70B8}   
\definecolor{skyblueE}{HTML}{0F3D63}   
\usepackage[most]{tcolorbox}
\usepackage{xcolor}
\usepackage[scaled]{helvet}

\definecolor{takeawayborder}{HTML}{111111} 
\definecolor{takeawaytitlebg}{HTML}{7A0F14} 
\definecolor{takeawaybg}{HTML}{FFFFFF}      

\newtcolorbox{takeawaybox}[1][]{
  enhanced,
  breakable,
  colback=takeawaybg,
  colframe=takeawayborder,
  boxrule=1.1pt,
  arc=2.2mm,
  left=2mm,
  right=2mm,
  top=1.3mm,
  bottom=1.3mm,
  title=Takeaway,
  coltitle=white,
  fonttitle=\bfseries\fontfamily{phv}\selectfont,
  fontupper=\normalfont,
  attach boxed title to top left={xshift=2mm,yshift=-2mm},
  boxed title style={
    colback=InterBlue,
    colframe=InterBlue,
    arc=1.3mm,
    boxrule=0pt
  },
  before skip=8pt,
  after skip=10pt,
  #1
}

\definecolor{Aone}{HTML}{1B9E77}   
\definecolor{Atwo}{HTML}{D95F02}   
\definecolor{Athree}{HTML}{7570B3} 
\definecolor{Afour}{HTML}{E7298A}  
\definecolor{Afive}{HTML}{66A61E}  
\definecolor{Asix}{HTML}{E6AB02}   
\definecolor{Aseven}{HTML}{A6761D} 
\definecolor{Aeight}{HTML}{666666} 

\title{Value-Gradient Hypothesis of RL for LLMs}

\author{
  \centerline{
    \begin{tabular}{cccc}
      \textbf{Arip Asadulaev}\thanks{Contact: \texttt{arip.asadulaev@mbzuai.ac.ae}} & \textbf{Daniil Ognev} & \textbf{Karim Salta} & \textbf{Martin Takac} \\
      \textnormal{MBZUAI} & \textnormal{MBZUAI} & \textnormal{Independent} & \textnormal{MBZUAI}
    \end{tabular}
  }
}

\iclrfinalcopy 

\usepackage[textsize=tiny]{todonotes}

\begin{document}

\maketitle

\begin{abstract}
Reinforcement learning substantially improves pretrained language models, but it remains understudied why critic-free methods such as PPO and GRPO work as well as they do, and when they should provide the largest gains. We develop a value-gradient perspective of critic-free RL for LLM post-training. First, under a differentiable rollout and additive-noise parameterization, we show that the actor update is value-gradient-like in expectation: \emph{the backward pass propagates costates whose conditional expectation equals the value gradient}. Second, for discrete transformer policies, we show that autodifferentiation through attention produces empirical costates that approximate this value signal, with an error controlled by the sampling gap and policy entropy. These results motivate a decomposition of RL impact into value gradient signal and reachable reward headroom, yielding a criterion for when RL should be most effective along a pretraining trajectory.


\end{abstract}

\section{Introduction}
\begin{wrapfigure}{r}{0.49\textwidth} 
    \vspace{-3mm}
    \centering
    \includegraphics[width=\linewidth]{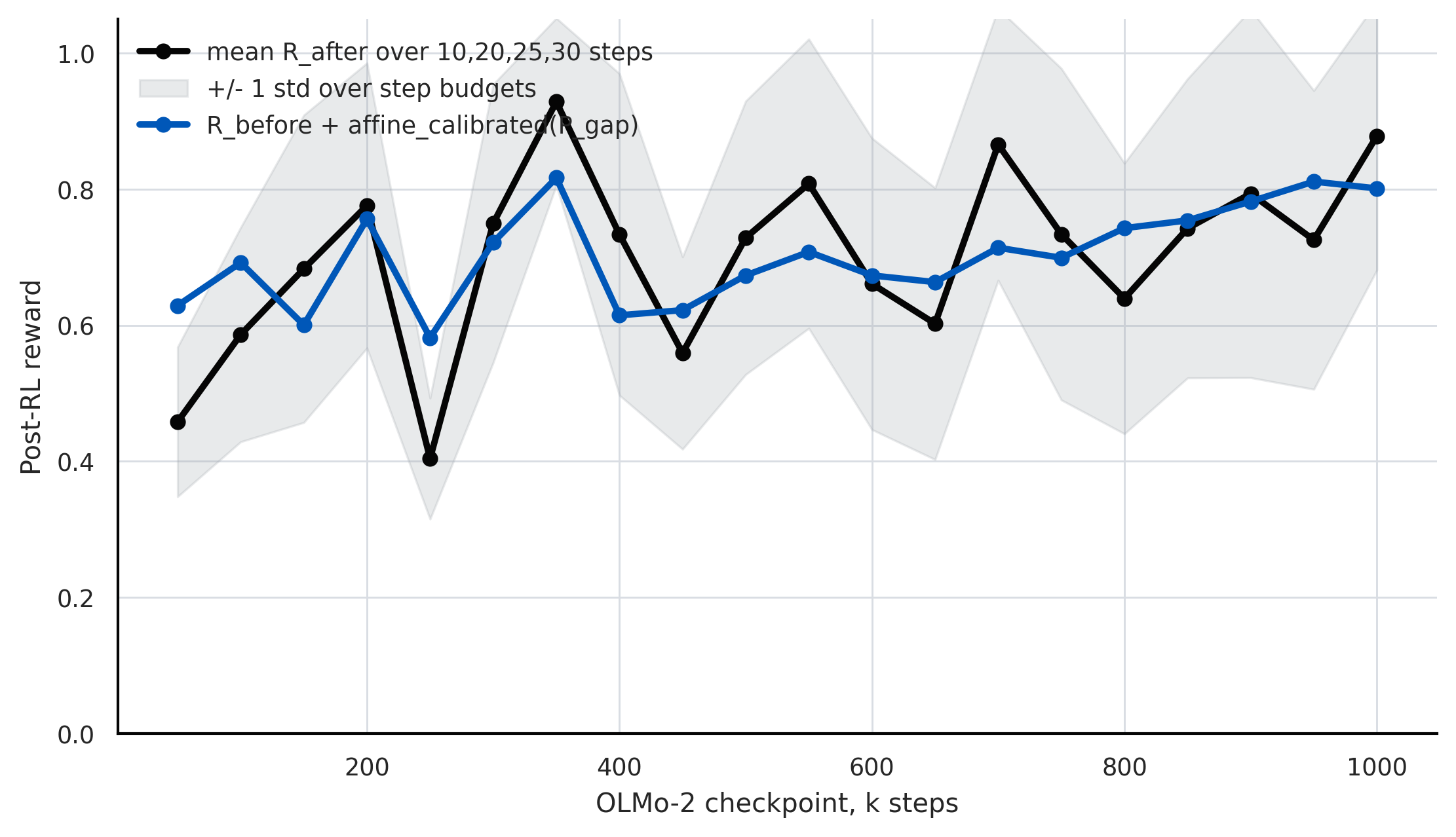}\\
    \vspace{0pt}
    \vspace{-3mm}
    \caption{Real RL gain vs. predicted one using value impact formula (Section \ref{sec:rl-readiness}, Eq. \ref{eq:readiness-summary}).}
    \label{fig:costates_25}
    \vspace{-5mm}
\end{wrapfigure} 
Recently, Large Language Models (LLMs) achieve state-of-the-art reasoning using Group Relative Policy Optimization (GRPO)~\citep{shao2024deepseekmath}, which discards the critic entirely, yet classical Reinforcement Learning (RL) theory predicts that critic-free methods should fail at long-horizon credit assignment. Why don't they? \emph{In this paper we argue that critic-free RL in LLMs is not value-free}. The central claim is that the actor backward pass already carries a value-gradient-like signal. In a differentiable rollout, this signal is exactly the costate propagated by Backpropagation Through Time (BPTT). In a discrete transformer, the same structure survives approximately because attention provides a differentiable pathway for credit transport around the token-sampling bottleneck.

First, in our paper, under a continuous relaxation \citep{fairbank2012value}, we show that the Proximal Policy Optimization (PPO)/GRPO actor update is value-gradient-like in expectation. Second, for discrete transformer policies, we show that the empirical costate computed by autodiff approximates the continuous BPTT signal, with an error controlled by the sampling gap and policy entropy. Third, we use this perspective to derive an RL-impact decomposition into usable value-gradient signal and reachable reward headroom, which predicts when RL should be most effective along pretraining (Figure \ref{fig:costates_25}). This perspective gives a concrete answer to a practical question: \emph{RL should help most at checkpoints that are simultaneously close enough to the value-gradient regime to transmit useful credit and far enough from saturation to retain reward-improving trajectories}. 

Our \textbf{contributions} are: We show that under a differentiable rollout and shift/additive-noise policy, the local GRPO actor update is value-gradient-like in expectation. We show that in transformers with discrete token sampling, the empirical costate computed by autodiff approximates the BPTT costate, with an error controlled by the sampling gap and attention-based credit transport. We derive a predictive RL-impact decomposition into usable value-gradient signal and reachable reward headroom, which empirically can be used for checkpoint selection during pretraining.

\section{Background}
\label{sec:background}
\textbf{Notation}. A prompt/question is denoted by $q\sim P(Q)$. Given $q$, a policy with
parameters $\theta$ generates an autoregressive completion
$o=(o_1,\dots,o_T)$ with token-level factorisation
$\pi_\theta(o\mid q)=\prod_{t=1}^T \pi_\theta(o_t\mid s_t)$, $s_t := (q,o_{<t})$, $a_t := o_t.$
We write $(s_t,a_t)_{t=1}^{T}$ for the token-level trajectory. Let
$r(s_t,a_t)\in\mathbb{R}$ denote a per-token reward. \emph{Outcome-only} RL is
the case where $r(s_t,a_t)=0$ for $t<T$ and $r(s_T,a_T)=r(q,o)$ is a terminal
reward. With a fixed discount $\gamma\in(0,1]$, the discounted return and the
return-to-go are $R := \sum_{t=1}^{T} \gamma^{t-1}\, r(s_t,a_t)$, $R_t := \sum_{j=t}^{T} \gamma^{j-t}\, r(s_j,a_j)$.

For any policy $\pi$, define the time-indexed value function, action-value,
and advantage $V^\pi_t(s):=\mathbb{E}[R_t\mid s_t=s]$, $Q^\pi_t(s,a):=\mathbb{E}[R_t\mid s_t=s,\,a_t=a]$, 
$A^\pi_t(s,a):=Q^\pi_t(s,a)-V^\pi_t(s)$. 
The \emph{value gradient} is the state-gradient of the value function,
$G_t^\pi(s)\;:=\;\frac{\partial V_t^\pi(s)}{\partial s}$.
Unlike the scalar $V_t^\pi$, it is a vector field on state space. Throughout the paper we write gradients with the partial-derivative symbol:
$\frac{\partial F(\theta)}{\partial\theta}$ for a gradient in parameter
space, $\frac{\partial\phi(s)}{\partial s}$ and
$\frac{\partial\phi(s,a)}{\partial a}$ for gradients in state or action
space, respectively.

\subsection{RL for LLMs}
GRPO~\citep{shao2024deepseekmath} is commonly presented as a
PPO~\citep{schulman2017proximal} variant that removes the learned critic/value
model. For each prompt $q$, it samples a group of $n$ completions
$\{o^i\}_{i=1}^{n}$ from an old policy $\pi_{\theta_{\mathrm{old}}}$, where
the index $i$ identifies the rollout (and its return, which may be
terminal-only or per-token). Let $T_i:=|o^i|$ and
$s_{i,t}:=(q,o^i_{<t})$. In the outcome-only case each completion receives a
scalar reward $r_i=r(q,o^i)$. GRPO forms a group-normalised signal
$\tilde r_i$ and the tokenwise likelihood ratio: for $t=1,\dots,T_i$,
$\tilde r_i := \frac{r_i-\mathrm{mean}(r)}{\mathrm{std}(r)+\xi_{\mathrm{num}}}$, $
\rho_{i,t}(\theta)
:= \frac{\pi_\theta(o^i_t\mid s_{i,t})}{\pi_{\theta_{\mathrm{old}}}(o^i_t\mid s_{i,t})},
$
with $\xi_{\mathrm{num}}>0$ for numerical stability and
$r=(r_1,\dots,r_G)$. GRPO then uses a tokenwise advantage estimate constant
along the completion,
$\widehat A_{i,t}:=\tilde r_i$ for all $t$ (process-supervision variants can
use token- or step-level reward-to-go). The objective is the GRPO clipped
surrogate with a KL penalty to a fixed reference policy $\pi_{\mathrm{ref}}$:
\begin{equation}\label{eq:grpo_objective}
\begin{aligned}
J(\theta) = \mathbb{E}\Bigg[\frac{1}{G}\sum_{i=1}^G \frac{1}{T_i}\sum_{t=1}^{T_i}
& \Bigg(\min\!\Bigl(\rho_{i,t}(\theta)\widehat A_{i,t},\;
  \operatorname{clip}(\rho_{i,t}(\theta),1-\varepsilon,1+\varepsilon)\widehat A_{i,t}\Bigr) -\beta\,{\mathrm{KL}}\!\Bigl(\pi_\theta\,\|\,\pi_{\mathrm{ref}}\Bigr)\Bigg)\Bigg],
\end{aligned}
\end{equation}
where $\varepsilon>0$ is the clipping threshold and $\beta>0$ controls the KL
regularisation. GRPO differs from PPO only in how $\widehat A_{i,t}$ is
constructed: it uses a group-normalised return (computed from multiple
rollouts for the same prompt) and typically holds the resulting scalar
constant along the trajectory.

\subsection{Gradient Estimators}
\label{sec:gradients}
Let $x$ be a random variable, $c$ a differentiable scalar cost, and
$F(\theta) := \mathbb{E}_x[c(x)]$ the objective. There are two canonical ways
$\theta$ can enter this expectation \citep{schulman2015gradient}. For an \textbf{Score-function (SF) estimator},
if $x \sim p(\,\cdot\,;\theta)$, then by the log-derivative trick
\begin{equation}
    \frac{\partial}{\partial\theta}\,\mathbb{E}_{x\sim p(\cdot;\theta)}\!\bigl[c(x)\bigr]
    \;=\; \mathbb{E}_{x}\!\left[c(x)\,\frac{\partial}{\partial\theta}\log p(x;\theta)\right].
    \label{eq:sf}
\end{equation}
This identity is valid whenever $p(x;\theta)$ is differentiable in $\theta$.
Crucially, $c$ need not be differentiable, or even continuous, in $x$.
That is exactly why SF is the natural estimator for classical RL:
discrete actions, non-differentiable rewards, and unknown dynamics are all
admissible. The cost is variance: the estimator uses only the \emph{scalar}
$c(x)$, not its slope, and therefore \emph{ignores all local geometry of $c$.}

\textbf{Pathwise-derivative (PD) estimator.}
If instead $x = x(z,\theta)$ is a differentiable function of $\theta$ and an
exogenous noise variable $z \sim p(z)$ whose distribution does \emph{not} depend
on $\theta$ (a \emph{reparameterization}), then differentiation and expectation
commute directly:
\begin{equation}
    \frac{\partial}{\partial\theta}\,\mathbb{E}_{z}\!\bigl[c(x(z,\theta))\bigr]
    \;=\; \mathbb{E}_{z}\!\left[\frac{\partial}{\partial\theta}c(x(z,\theta))\right].
    \label{eq:pd}
\end{equation}
PD exploits $\frac{\partial c}{\partial x}$ directly and is typically lower variance
than SF when both apply \citep{rezende2014stochastic}. The price is a stronger
regularity requirement: $c\circ x(\cdot,\theta)$ must be (almost everywhere)
differentiable, and the sampling must admit a reparameterization.

\subsection{Costates and Value-gradients}
\label{sec:bptt}

We now lift the single-variable policy gradient setup of \S\ref{sec:gradients} to trajectories and ask what object the backward pass in RL settings computes.

\begin{definition}[Differentiable rollout]
\label{def:reparam}
The functions $\pi_\theta, f_\theta, r$ are differentiable in their arguments.
The noise law $p(\xi)$ does not depend on $\theta$, and
\begin{equation}
    a_t = \pi_\theta(s_t,\xi_t),\qquad s_{t+1} = f_\theta\bigl(s_t, \pi_\theta(s_t,\xi_t)\bigr),
    \qquad \xi_t\sim p(\xi)\text{ i.i.d., independent of }\theta.
    \label{eq:diffrollout}
\end{equation}
\end{definition}
Let $D$ denote the total derivative with respect to the state, accounting for both the direct state dependence and the indirect dependence through the policy's action. $D f_\theta(s_t,a_t) = \frac{\partial f_\theta}{\partial s} + \frac{\partial f_\theta}{\partial a}\frac{\partial \pi_\theta}{\partial s}$ for the vector-valued dynamics, and similarly $D r$ for the reward. Because all randomness is exogenous under Definition~\ref{def:reparam}, the sampled return $R_1(\theta,\xi_{1:T})$ is a deterministic function of the parameters given the noise. This allows differentiation to commute with expectation, resulting in a pathwise identity formula analogous to \eqref{eq:pd}: $\frac{\partial J(\theta)}{\partial \theta} = \mathbb{E}\!\left[\frac{\partial R_1}{\partial \theta}\right]$. The gradient $\frac{\partial R_1}{\partial \theta}$ is computed by differentiating the unrolled computation graph~\eqref{eq:diffrollout}, a process known as \emph{backpropagation through time} (BPTT) \citep{fairbank2012value}. Crucially, BPTT does not propagate the parameter gradient itself. It propagates the state-sensitivity adjoint, which we call the \textbf{costate}:
\begin{equation}\label{eq:lambda}
    \lambda_t \;:=\; \frac{\partial R_t}{\partial s_t}, \qquad \lambda_{T+1}:=0.
\end{equation}
Intuitively, the effect of the current state on future return has two pieces:
the immediate effect on the current reward, plus the future effect pushed
backward through the transition Jacobian. Formally:

\begin{proposition}[Adjoint recursion, \citep{fairbank2012value}]
\label{prop:adjoint}
Under Definition~\ref{def:reparam}, the costates satisfy
\begin{equation}
    \boxed{\;\lambda_t \;=\; D r(s_t, a_t) \;+\; \gamma\,\bigl(D f_\theta(s_t,a_t)\bigr)^{\!\top}\!\lambda_{t+1}\;}
    \label{eq:adjoint}
\end{equation}
for $t = T,\ldots,1$, and the exact pathwise parameter gradient is
\begin{equation}
\begin{split}
\frac{\partial J(\theta)}{\partial\theta}=\mathbb{E}\Bigg[\sum_{t=1}^{T}\gamma^{t-1}\Bigg(&\gamma\left(\frac{\partial f_{\theta}(s_{t},a_{t})}{\partial\theta}\right)^{\top}\lambda_{t+1} \\
+&\left(\frac{\partial\pi_{\theta}(s_{t},\xi_{t})}{\partial\theta}\right)^{\top}\left(\frac{\partial r(s_{t},a_{t})}{\partial a_{t}}+\gamma\left(\frac{\partial f_{\theta}(s_{t},a_{t})}{\partial a_{t}}\right)^{\top}\lambda_{t+1}\right)\Bigg)\Bigg].
\end{split}
\label{eq:bptt-grad}
\end{equation}
\end{proposition}
The proof (see Appendix~\ref{app:proofs}) is a direct chain-rule expansion of
$R_t = r(s_t,a_t)+\gamma R_{t+1}$.

\textbf{Costates are value-gradient estimators.}
Conditioning on $s_t=s$, the value function satisfies $V_t^\pi(s) = \mathbb{E}\!\left[r(s_t,a_t)+\gamma V_{t+1}^\pi(s_{t+1}) \,\middle|\, s_t=s\right]$. Comparing this value recursion with the sampled costate recursion in \eqref{eq:adjoint}, we see that the two have the same form. More precisely, differentiating \ref{eq:adjoint} with respect to $s$ and interchanging differentiation and expectation (valid under Definition~\ref{def:reparam}) gives
\begin{equation}
\frac{\partial V_t^\pi(s)}{\partial s}
=
\mathbb{E}\!\left[
D r(s_t,a_t)
+
\gamma \bigl(D f_\theta(s_t,a_t)\bigr)^{\!\top}
G_{t+1}^\pi(s_{t+1})
\,\middle|\, s_t=s
\right].
\label{eq:value-gradient-recursion}
\end{equation}
The value gradient is the corresponding \emph{conditional expectation} over that
future noise. Therefore,
\begin{equation}
\mathbb{E}[\lambda_t \mid s_t=s]
=
\frac{\partial V_t^\pi(s)}{\partial s}
\label{eq:costate-is-vg}
\end{equation}
Thus the quantity propagated backward by BPTT is, at each step, a Monte Carlo
sample of the value gradient. Thus
\emph{critic-free} does not mean \emph{value-free}: the relevant value information is
present as a propagated gradient signal rather than as a separately fitted
scalar critic. 

\begin{takeawaybox}
The key object for this paper is the costate: in a differentiable rollout, the backward
signal propagated by BPTT is a Monte Carlo estimator of the value gradient.
\end{takeawaybox}

\section{Continuous Lens: Why Critic-Free RL Is Value-Gradient-Like}
\label{sec:connect_grpo_vg}

We now show that the GRPO/PPO actor update equals, in expectation, the BPTT pathwise gradient of a continuous-relaxed rollout. The standard GRPO/PPO derivation is written for discrete token actions, $s_t := (q,o_{<t})$, $a_t := o_t$, $s_{t+1}=(s_t,a_t)$,
where the sampling step blocks differentiation through the trajectory. However, \cite{fairbank2012value} integration-by-parts equivalence supplies the missing conceptual link: under an additive-noise parameterization, the expected score-function update can be rewritten as an action-derivative (PD-like) update. To make
the backward signal explicit, we introduce a continuous relaxation of the LLM rollout. At the level of hidden states, logits, or other differentiable
surrogate states, we apply Definition~\ref{def:reparam}  with differentiable $f_\theta$, $\pi_\theta$, and $r$. This turns the rollout
into a reparameterized computation graph, so the pathwise analysis of
\S\ref{sec:bptt} applies. The only ingredient we need is the score-function/pathwise bridge.

\begin{lemma}[Fairbank’s \citep{fairbank2012value} SF $=$ PD under a shift policy]
\label{lem:shift}
Fix a state $s$ and let
$F(\theta)=\mathbb{E}_{a\sim\pi_\theta(\cdot\mid s)}[r(s,a)]$,
with $r$ differentiable in $a$. If
\[
\pi_\theta(a\mid s)=\nu\bigl(a-\bar a_\theta(s)\bigr),
\]
for a differentiable density $\nu$ with vanishing boundary terms, then the
score-function and pathwise forms agree in expectation.
In particular,
\begin{equation}
\underbrace{\mathbb{E}\!\left[r(s,a)\,\frac{\partial}{\partial\theta}\log\pi_\theta(a\mid s)\right]}_{\text{SF form}}
\;=\;
\underbrace{\Bigl(\frac{\partial \bar a_\theta(s)}{\partial\theta}\Bigr)^{\!\top}\,\mathbb{E}\!\left[\frac{\partial r(s,a)}{\partial a}\right]}_{\text{PD form}}.
\label{eq:sf-pd-equiv}
\end{equation}
\end{lemma}

Lemma~\ref{lem:shift} is the bridge: it guarantees that
the score-function update used in practice matches a pathwise update in
expectation. Applied token-by-token, under a
shift/additive-noise parameterisation
$\pi_\theta(a_t\mid s_t)=\nu\!\bigl(a_t-\bar a_\theta(s_t)\bigr)$, if the GRPO weight $\widehat A_{i,t}$ is treated as a stop-gradient scalar,
each local score-function term admits:
\begin{equation}\label{eq:weighted_sf_to_pd}
\mathbb{E}\!\left[
\widehat A_{i,t}\,r(s_t,a_t)\,
\frac{\partial}{\partial\theta}\log\pi_\theta(a_t\mid s_t)
\right]
\;=\;
\Bigl(\frac{\partial\bar a_\theta(s_t)}{\partial\theta}\Bigr)^{\!\top}\!
\mathbb{E}\!\left[
\frac{\partial\,\bigl(\widehat A_{i,t}\,r(s_t,a_t)\bigr)}{\partial a_t}
\right].
\end{equation}
The right-hand side is PD-shaped: it depends on the slope of the
action-to-reward signal, not on the score function. No SF
information is lost and expectation is unchanged. Then, under the continuous rollout \eqref{def:reparam}, the backward quantity propagated by GRPO is the costate, and, by \S\ref{sec:bptt}, it satisfies
\begin{equation}
\lambda_t := \frac{\partial R_t}{\partial s_t}, \quad \mathbb{E}[\lambda_t \mid s_t=s]
=
\frac{\partial V_t^\pi(s)}{\partial s}.
\label{eq:costate}
\end{equation}
Thus, in the continuous surrogate, the GRPO/PPO update is value-gradient-like
in expectation: even without a fitted critic, the backward pass carries a
Monte Carlo estimator of the state-gradient of future return. The full
derivation is deferred to Appendix~\ref{app:proofs}.
\begin{takeawaybox}
Under a differentiable continuous relaxation, PPO/GRPO is value-gradient-like in expectation,
so critic-free actor updates can still carry a principled credit-assignment signal without a learned critic.
\end{takeawaybox}


\section{Discrete GRPO Sends Approximate BPTT Signals}
\label{sec:discrete_grpo}
Section \S\ref{sec:connect_grpo_vg} established, under a continuous relaxation
and a shift/additive-noise policy, that the local GRPO gradient is a
value-gradient update (Corollary~\ref{cor:grpo_is_vg}). In practice, LLM
rollouts use a \emph{discrete} categorical policy. The key realisation is that when one computes
$\frac{\partial L_{\mathrm{GRPO}}}{\partial\theta}$ by standard
autodifferentiation, one is already backpropagating through the
transformer's internal computation graph across positions. The attention
mechanism creates differentiable pathways between positions; the only place
differentiability breaks is at the token-sampling boundary. The question thus
becomes: \emph{how much of the BPTT costate structure of \S\ref{sec:bptt}
survives despite the discrete sampling gaps?}

For a single trajectory, the log-probability at position $t$ depends on the final-layer hidden state, and the GRPO loss (locally, ignoring clipping) is
\begin{equation}\label{eq:grpo-loss}
    L(\theta) = \sum_{t=1}^{T} \widehat A_t \cdot \ell_t(\theta),
    \qquad \ell_t := \log \pi_\theta(o_t\mid s_t) :=  \log  \softmax\!\bigl(W_{\mathrm{head}}\cdot h_t^{(L)}(\theta)\bigr)_{o_t},
\end{equation}
and in a transformer, $h_t^{(L)}$ depends on hidden representations at
\emph{all previous positions} through the attention mechanism,
$h_t^{(l)} = \mathrm{TransformerBlock}^{(l)}\!\bigl(h_{1:t}^{(l-1)}\bigr)$, $l = 1,\ldots,L.$
Consequently, there exist differentiable Jacobians $J_{t\leftarrow t'} := \frac{\partial h_t^{(L)}}{\partial h_{t'}^{(L)}}$, $t'\le t$, computed through the multi-layer attention stack. These are \emph{exactly the transition Jacobians} that play the role of $Df_\theta$ in \S\ref{sec:bptt}.

\begin{definition}[Empirical costate]\label{def:empirical-costate}
The empirical costate at position $t$ is the gradient of the
advantage-weighted loss with respect to the final-layer hidden state:
\begin{equation}\label{eq:empirical-costate}
    \hat\lambda_t := \frac{\partial}{\partial h_t^{(L)}} \sum_{k\ge t} \widehat A_k \cdot \ell_k.
\end{equation}
\end{definition}

By the chain rule through attention, $\hat\lambda_t = \widehat A_t\cdot\frac{\partial\ell_t}{\partial h_t^{(L)}} +\sum_{k>t}\widehat A_k\cdot\frac{\partial\ell_k}{\partial h_k^{(L)}}\cdot J_{k\leftarrow t}$, or equivalently, in recursive form,
\begin{equation}\label{eq:costate-recursion}
    \boxed{\;
      \hat\lambda_t \;=\; \widehat A_t\cdot\frac{\partial\ell_t}{\partial h_t^{(L)}}
      \;+\; J_{t+1\leftarrow t}^{\!\top}\cdot\hat\lambda_{t+1}
    \;},\qquad \hat\lambda_{T+1}:=0.
\end{equation}

This is \textbf{structurally identical} to the BPTT costate recursion of
Proposition~\ref{prop:adjoint}, $\lambda_t \;=\; D r(s_t,a_t) \;+\; \gamma\,\bigl(Df_\theta(s_t,a_t)\bigr)^{\!\top}\lambda_{t+1}$. The correspondence is
\begin{align}
    D r &\;\longleftrightarrow\; \widehat A_t\cdot\frac{\partial\ell_t}{\partial h_t^{(L)}}
        && \text{(immediate credit signal)},\label{eq:corr1}\\
    (Df_\theta)^{\!\top} &\;\longleftrightarrow\; J_{t+1\leftarrow t}^{\!\top}
        && \text{(transition Jacobian through attention)},\label{eq:corr2}\\
    \lambda_{t+1} &\;\longleftrightarrow\; \hat\lambda_{t+1}
        && \text{(propagated future signal)}.\label{eq:corr3}
\end{align}

\subsection{What is missing vs.\ exact BPTT}

In the continuous relaxation of \S\ref{sec:bptt}, the transition
$s_{t+1}=f_\theta(s_t,a_t)$ is fully differentiable, so $Df_\theta$ captures
\emph{everything}, including how a perturbation of $s_t$ would change the
continuous action $a_t$, which would change $s_{t+1}$. In discrete GRPO, there is a \textbf{non-differentiable gap} at each sampling
step:
\begin{equation}\label{eq:chain}
    h_t^{(L)}
    \;\xrightarrow{\;\text{diff}\;}\; z_t
    \;\xrightarrow{\;\text{sample}\;}\; o_t
    \;\xrightarrow{\;\text{embed}\;}\; e_t
    \;\xrightarrow{\;\text{diff}\;}\; h_{t+1}^{(L)}.
\end{equation}
The Jacobian $J_{t+1\leftarrow t}^{\attn}$ that autodiff computes goes through
the attention pathway: how $h_t^{(L)}$ influences $h_{t+1}^{(L)}$ through
attention at position $t+1$, treating the input token $o_t$ as fixed. It does
\emph{not} capture the path
$h_t^{(L)} \to z_t \to o_t \to e_t \to h_{t+1}^{(L)}$ through the sampled
token.

The exact BPTT Jacobian therefore decomposes as
\begin{equation}\label{eq:jacobian-decomp}
    Df_\theta^{\mathrm{exact}}
    \;=\;
    \underbrace{J_{t+1\leftarrow t}^{\attn}}_{\text{what GRPO computes}}
    \;+\;
    \underbrace{\frac{\partial h_{t+1}^{(L)}}{\partial e_t}
      \cdot \frac{\partial e_t}{\partial o_t}
      \cdot \frac{\partial o_t}{\partial z_t}
      \cdot \frac{\partial z_t}{\partial h_t^{(L)}}}_{\text{missing: through the sampling step}}.
\end{equation}
The second term is non-differentiable in the discrete case. \textbf{This is
the only gap.}

\begin{proposition}[Sampling-gap bound]\label{prop:sampling-gap}
The missing Jacobian term is small when the policy is near-deterministic.
Specifically, if the policy entropy at position $t$ is $H_t$, then the
effective Jacobian through the sampling step (in a straight-through or
Gumbel-Softmax sense) has operator norm bounded by
\begin{equation}\label{eq:gap-bound}
    \bigl\| Df_\theta^{\mathrm{exact}} - J_{t+1\leftarrow t}^{\attn}\bigr\|
    \;\leq\; C\cdot\sqrt{\frac{H_t}{\log|V|}},
\end{equation}
where $|V|$ is the vocabulary size and $C$ depends on embedding norms.
\end{proposition}

\begin{remark}[Why this bound is natural]
When the policy is near-deterministic ($H_t\approx 0$), the softmax is peaked
at one token. A small perturbation of the logits does not change which token
is sampled (the argmax is stable), so $\partial o_t/\partial z_t\approx 0$ in
a distributional sense. The straight-through estimator makes this precise:
the expected Jacobian through sampling scales with the ``softness'' of the
distribution.
\end{remark}

\begin{remark}[Why this matters for LLMs]
After pretraining, LLM policies have relatively low entropy on most tokens
(especially for common syntactic tokens, function words, code syntax). The
KL constraint in GRPO keeps the policy close to this low-entropy reference,
so on most of the trajectory the missing signal is small.
\end{remark}
\begin{wrapfigure}{r}{0.49\textwidth} 
    \vspace{-5mm}
    \centering
    \includegraphics[width=\linewidth]{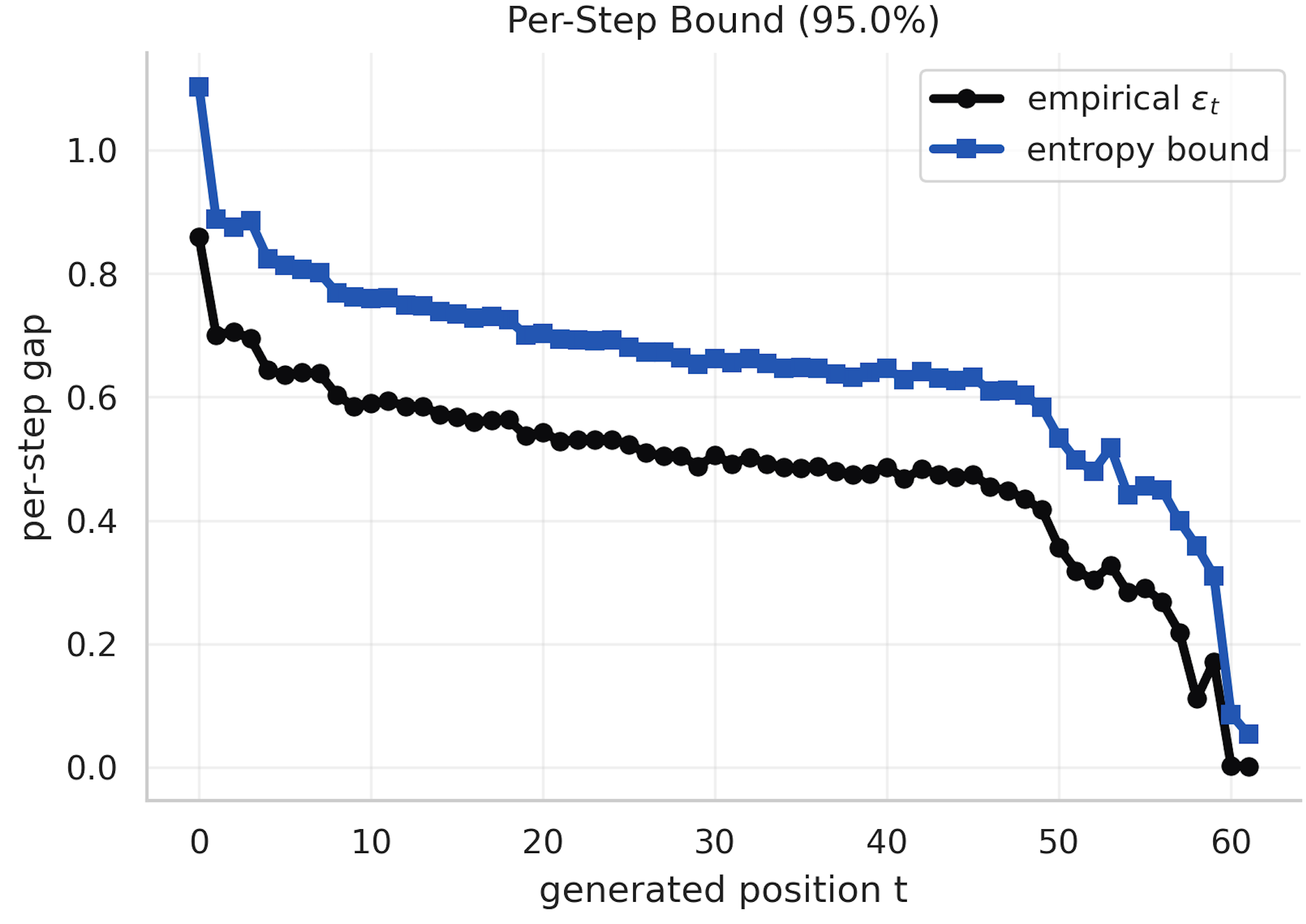}
    \caption{Bound inequality plot. Real value gradient gap vs proposed bound (Section \ref{sec:discrete_grpo}, Prop. \ref{prop:sampling-gap}).}
    \label{fig:bound}
    \vspace{-5mm}
\end{wrapfigure}
\textbf{Attention is doing the heavy lifting}. In an RNN, $h_{t+1}$ depends on $h_t$ only through the recurrence
$h_{t+1}=f(h_t, e_{o_t})$. The \emph{only} pathway is through the discrete
token, so the sampling gap blocks \emph{all} temporal credit flow. In a transformer, $h_{t+1}^{(L)}$ depends on $h_t^{(L)}$ through
\textbf{attention}, a direct, differentiable, content-based pathway that
does not go through the sampling step at position $t$. The attention scores
\begin{equation}
    \alpha_{t+1,\,t'} \;\propto\; \exp\!\bigl(q_{t+1}\cdot k_{t'}\bigr)
\end{equation}
create soft, differentiable pointers to all previous hidden states. This
gives transformers a \emph{bypass around the discrete sampling bottleneck}:
credit can flow from future positions to past hidden states through attention
without differentiating through token sampling.

\begin{proposition}[Attention-pathway rank and magnitude]\label{prop:attn-rank}
For a transformer with $L$ layers and causal attention with
$n_{\mathrm{heads}}$ heads of dimension $d_{\mathrm{head}}$, the
attention-pathway Jacobian $J_{t+1\leftarrow t'}^{\attn}$ has rank at most
$L\cdot d_{\mathrm{head}}\cdot n_{\mathrm{heads}}$ per layer of attention,
and its magnitude is controlled by the attention weights $\alpha_{t,t'}$.
When attention is broadly distributed over context, the Jacobian carries
richer temporal credit information, partially compensating for the missing
sampling-path Jacobian.
\end{proposition}

\subsection{Approximation theorem}
\begin{theorem}[Discrete GRPO costates approximate BPTT costates]\label{thm:main}
Let $\hat\lambda_t$ be the empirical costate~\eqref{eq:empirical-costate} computed by
autodifferentiation through the GRPO loss, and let $\lambda_t$ be the exact
BPTT costate from Proposition~\ref{prop:adjoint} applied to the relaxed
rollout of \S\ref{sec:connect_grpo_vg}. Then
\begin{equation}\label{eq:main-bound}
    \bigl\|\mathbb{E}[\hat\lambda_t\mid h_t] - \mathbb{E}[\lambda_t\mid h_t]\bigr\|
    \;\leq\;
    \sum_{k=t+1}^{T}\gamma^{\,k-t}\cdot
    \underbrace{\bigl\|J_{k\leftarrow k-1}^{\attn}\bigr\|^{k-t-1}}_{\text{Jacobian chain growth}}
    \cdot
    \underbrace{\epsilon_k}_{\text{per-step sampling gap}},
\end{equation}
where $\epsilon_k$ is the per-step approximation error from the missing
sampling Jacobian, bounded by~\eqref{eq:gap-bound}.
\end{theorem}

Combining Theorem~\ref{thm:main} with identity~\eqref{eq:costate-is-vg},
\[
    \mathbb{E}[\hat\lambda_t\mid h_t] \;\approx\; \mathbb{E}[\lambda_t\mid h_t] \;=\; G_t^\pi(h_t),
\]
so the empirical costate is, up to the Theorem~\ref{thm:main} error, a Monte
Carlo estimator of the value gradient in the transformer's hidden-state
space. The error accumulates through the Jacobian chain but is controlled by \emph{policy entropy at each step}, and spectral properties of the attention Jacobians. Figure~\ref{fig:bound} evaluates the entropy bound \ref{prop:sampling-gap} formula, see more details in \S\ref{sec:experiments}.

\begin{takeawaybox}
In discrete transformers, autodiff already propagates an empirical costate through attention,
so critic-free RL remains close to the BPTT value-gradient picture despite discrete token sampling.
\end{takeawaybox}

\section{RL Impact Law Hypotheses}
\label{sec:rl-readiness}
The previous sections explain that the actor backward pass propagates an approximate hidden-state value-gradient signal, and the gap to the ideal value-gradient regime is controlled by the discrete sampling step together with attention-based credit transport. We now turn this into a predictive theory of \emph{RL readiness}. The key idea is simple: RL should be effective when the actor update contain a usable value-gradient signal. This gives a direct prediction for how much gain RL should produce from a given pretrained checkpoint. For any critic-free normalized policy-gradient method \(m\), let \(\hat\lambda_t^{(m)}\) denote the empirical hidden-state costate produced by the actor backward pass, and let \(G_t\) denote the hidden-state value gradient. The quantity
\begin{equation}
\mathbb E_{q,\tau,t}
\left[
\left\langle
\mathbb E[\hat\lambda_t^{(m)} \mid h_t],\,
G_t
\right\rangle
\right]
\label{eq:aligned-signal}
\end{equation}
measures how much useful RL signal the checkpoint actually sends. Using $\langle u,v\rangle = \|v\|_2^2 + \langle u-v,v\rangle$, we obtain the lower bound
\begin{equation}
\mathbb E_{q,\tau,t}
\left[
\left\langle
\mathbb E[\hat\lambda_t^{(m)} \mid h_t],\,
G_t
\right\rangle
\right]
\ge
\Sigma(\theta)-\Lambda_m(\theta),
\label{eq:signal-lower-bound}
\end{equation}
where
\begin{equation}
\Sigma(\theta)
:=
\mathbb E_{q,\tau,t}\!\left[\|G_t\|_2^2\right],
\qquad
\Lambda_m(\theta)
:=
\mathbb E_{q,\tau,t}
\left[
\|G_t\|_2\,
\left\|
\mathbb E[\hat\lambda_t^{(m)} \mid h_t]-G_t
\right\|_2
\right].
\label{eq:sigma-lambda}
\end{equation}
We therefore define the \emph{usable value-gradient signal}
\begin{equation}
S_m(\theta)
:=
\left(\Sigma(\theta)-\Lambda_m(\theta)\right)
\label{eq:usable-signal}
\end{equation}
Intuitively, \(\Sigma(\theta)\) measures how much value-gradient signal is available, while \(\Lambda_m(\theta)\) measures how much of that signal is lost because the actual actor update is still separated from the value-gradient regime. By the discrete approximation result of \S\ref{sec:discrete_grpo}, the gap term inside \(\Lambda_m(\theta)\) becomes small when the entropy-controlled sampling gap is small and attention transports credit effectively. Signal quality alone is not enough. Even a clean RL signal will produce little gain if the current policy has no useful reward headroom left to exploit. We therefore define the \emph{reachable headroom}
\begin{equation}
\mathcal H_\alpha(\theta)
:=
\mathbb E_{q \sim P(Q)}
\left[
\frac{1}{\alpha}
\log
\mathbb E_{\tau \sim \pi_\theta(\cdot \mid q)}
\left[
e^{\alpha R(\tau)}
\right]
-
\mathbb E_{\tau \sim \pi_\theta(\cdot \mid q)}
\left[
R(\tau)
\right]
\right],
\label{eq:headroom}
\end{equation}
where \(R(\tau)\) is the trajectory return. This quantity measures how much reward can still be gained by reweighting the model's existing trajectory distribution. It is small both when the model is too weak to sample any good trajectories and when it is already close to saturation; it is large only when better trajectories are already present in support and can still be amplified by RL. These two terms combine into the central predictive statement of the framework:
\begin{equation}
\boxed{
\text{RL Impact}
\propto
\underbrace{S_m(\theta)}_{\text{value-gradient signal}}
\times
\underbrace{\mathcal H_\alpha(\theta)}_{\text{reachable headroom}}}.
\label{eq:readiness-summary}
\end{equation}
For a fixed task, reward, and RL budget \(B\), we write this in the simplest one-constant form as
\begin{equation}
\Delta \mathrm{Perf}_{m,B}(\theta)
\approx
\kappa_{m,B}\,
S_m(\theta)\,
\mathcal H_\alpha(\theta),
\qquad
\kappa_{m,B}>0.
\label{eq:predictive-law}
\end{equation}
Eq.~\eqref{eq:readiness-summary} is the main prediction of the theory. RL should work best when the checkpoint is already close enough to the value-gradient regime to provide a usable actor signal, but still has enough reward headroom for RL to exploit. Practically, if \(\{\theta_N\}\) is a sequence of pretrained \emph{checkpoints}, then the predicted best point to start RL is
\begin{equation}
N^\star
=
\arg\max_N
S_m(\theta_N)\,
\mathcal H_\alpha(\theta_N).
\label{eq:best-switch-point}
\end{equation}
\begin{takeawaybox}
RL gain should scale with two factors at once: how much usable value-gradient signal the model transmits,
and how much reward-improving headroom remains in its trajectory distribution. 
\end{takeawaybox}

\section{Experiments}
\label{sec:experiments}
\begin{figure}[h!] 
    \vspace{-5mm}
    \centering
    \includegraphics[width=0.51\linewidth]{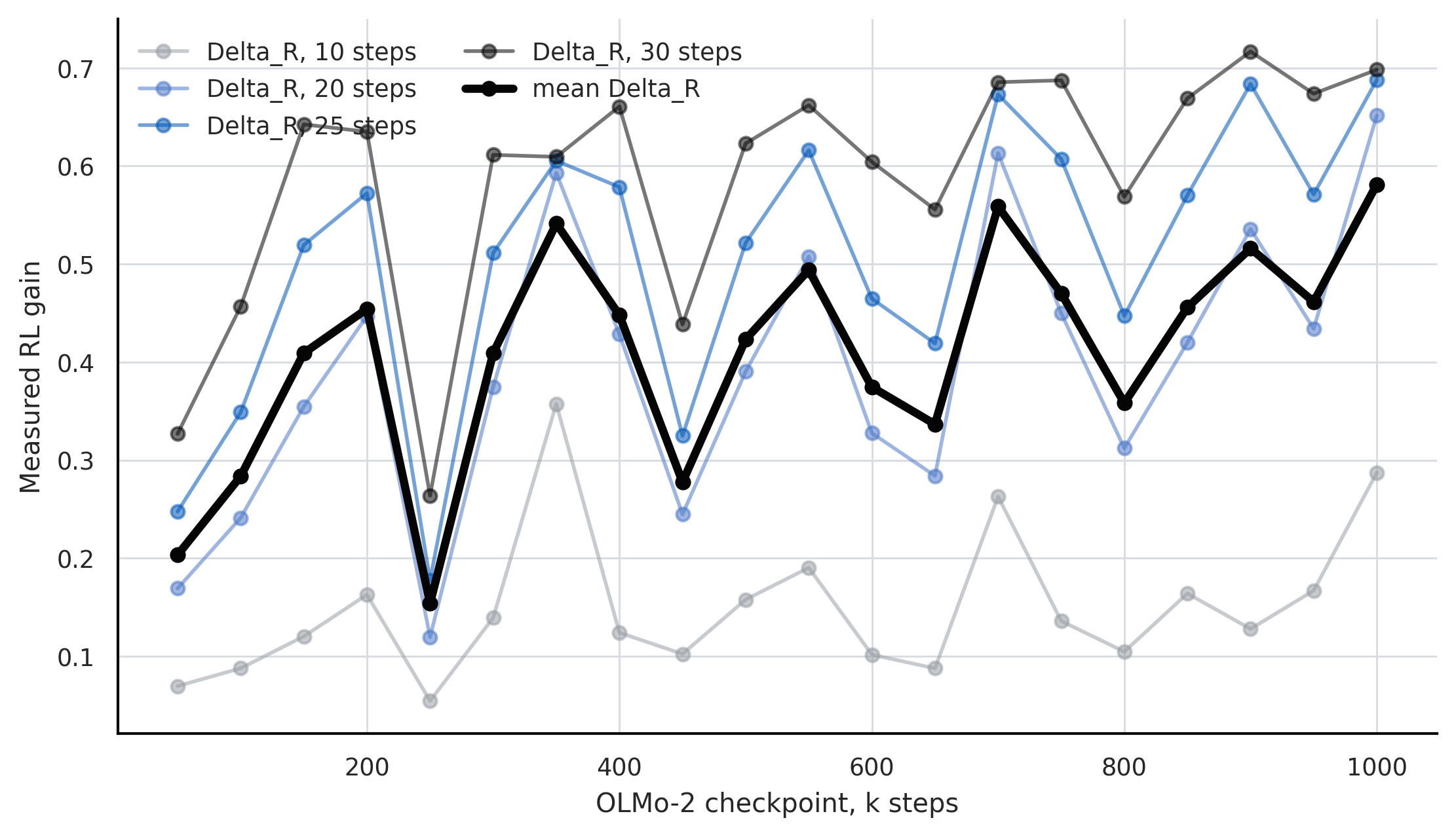}
    \vspace{0pt}
    \includegraphics[width=0.47\linewidth]{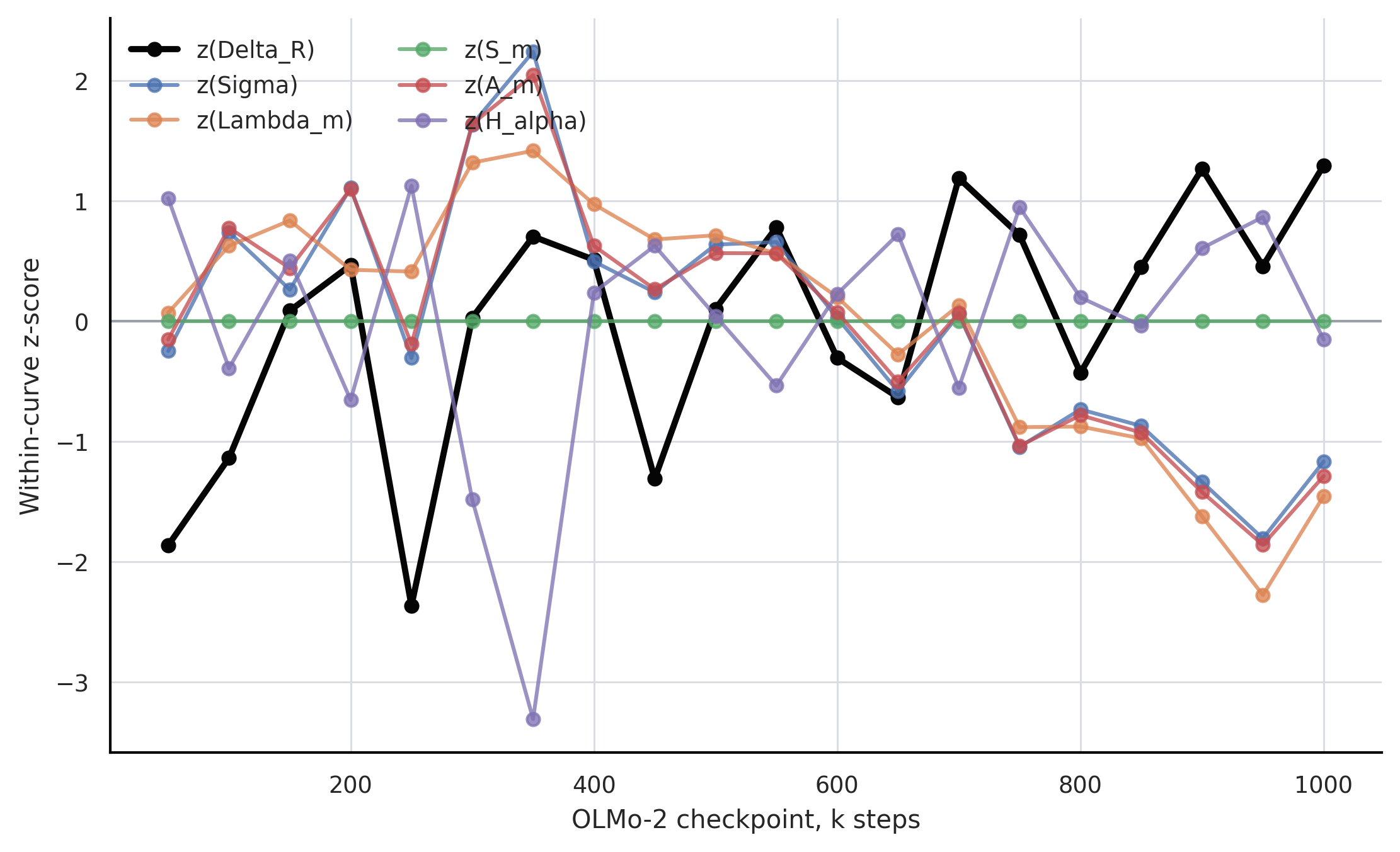}
    \vspace{-3mm}
    \caption{Results of the RL upon various OLMO-2 pretraining checkpoints. The left image shows the consistency of the achieved gains via RL using a different number of GRPO steps. The right image shows the behavior of the different components of our study.}
    \label{fig:components_gap}
    \vspace{-5mm}
\end{figure} 
We evaluate the proposed theory in two complementary settings. First, we test the entropy-controlled approximation bound from \S4 on OLMo-2~\citep{olmo20242olmo2furious} checkpoints. Second, we test whether the costate-based impact score predicts which checkpoints benefit most from a RL. 
\begin{wrapfigure}{r}{0.49\textwidth} 
    \vspace{-1mm}
    \centering
    \includegraphics[width=\linewidth]{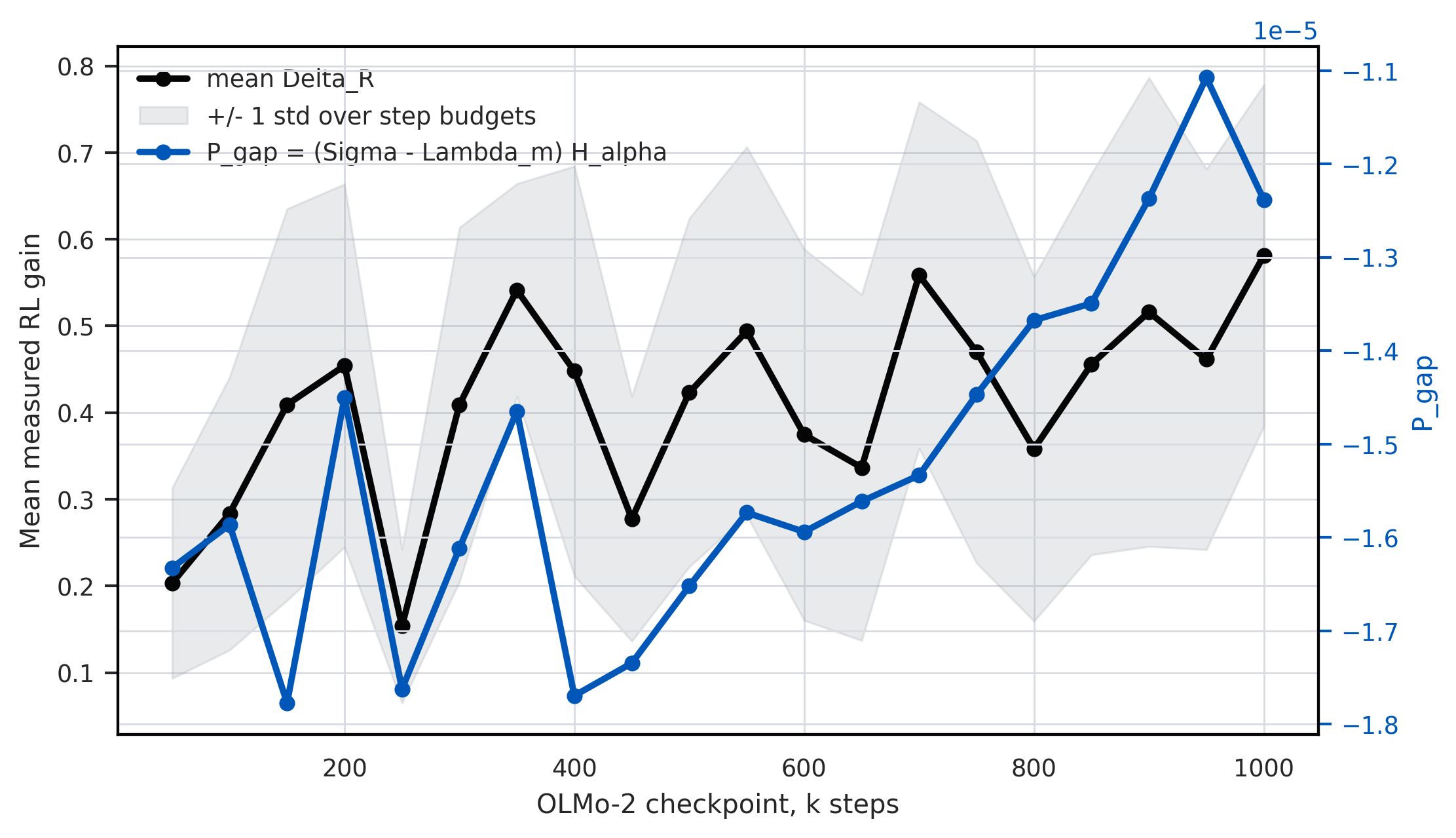}
    \caption{Real RL gain vs. predicted one using value impact formula (Section \ref{sec:rl-readiness}, Eq. \ref{eq:readiness-summary}).}
    \label{fig:rl_gain_vs_real}
    \vspace{-2mm}
\end{wrapfigure}
\textbf{Closed-form RL task.}
For the RL-impact experiment, we use OLMo-2 1B checkpoints from pretraining steps $50$k to $1$M in increments of $50$k. The task is a controlled label-copying problem. Given a prompt containing a target label in $\{A,B,C,D\}$, the model must put probability mass on the matching answer token. The reward is $R(\theta;q)=p_\theta(y^\star \mid q)$, where $y^\star$ is the correct label. This reward is differentiable in the model logits and avoids decoded-text, classifier, or parser discontinuities. For each checkpoint $\theta_N$, we measure the pre-RL reward $R_{\rm before}(\theta_N)=\mathbb{E}_{q}\,p_{\theta_N}(y^\star\mid q)$, then run GRPO from the checkpoint and measure the final post-RL reward $R_{\rm after}(\theta_N,K)$, where $K$ is the number of RL updates. We use several RL budgets, $K\in\{10,20,25,30\}$, and report both individual-budget results and the averaged response $\overline{\Delta R}(\theta_N)=\mathbb{E}_{K}\left[R_{\rm after}(\theta_N,K)-R_{\rm before}(\theta_N)\right]$.

\textbf{Entropy-bound verification.}
Figure~\ref{fig:bound} evaluates the entropy bound from \S\ref{sec:discrete_grpo}. In this experiment, we exactly computed our Proposition \ref{prop:sampling-gap} formula, comparing the real right-hand side $\epsilon$ with the entropy-based left side. For the OLMO-2 model, we estimate the costate approximation error and compare it to the entropy-controlled upper bound predicted by the theory, using the reward function described above. The purpose of this experiment is to verify our claim that the discrete-sampling gap is controlled by entropy.

\textbf{Costate-based predictors.}
For each checkpoint, we compute the true pathwise costate $G_t = \partial R/\partial h_t$ and the detached-reward RL costate $\widehat{\lambda}_t=\partial[\widehat{A}_t \log \pi_\theta(a_t\mid s_t)]/\partial h_t$, using the same prompts, hidden states, sampled actions, and reward evaluations. This matched-trajectory computation avoids mixing the relaxed and discrete distributions. We then estimate $\Sigma=\mathbb{E}\|G_t\|^2$, $\Lambda_m=\mathbb{E}[\|G_t\|\,\|\widehat{\lambda}_t-G_t\|]$, and the headroom term $H_\alpha=\frac{1}{\alpha}\log \mathbb{E}[\exp(\alpha R)]-\mathbb{E}[R]$. The impact score is $P_{\rm gap}=(\Sigma-\Lambda_m)H_\alpha$, which preserves the ordering information. 

\textbf{RL-response prediction.}
Figure~\ref{fig:costates_25}, \ref{fig:components_gap} and \ref{fig:rl_gain_vs_real} shows the costate components and the impact predictor. The component plot shows that the raw quantities vary systematically across checkpoints, while the predictor tracks the broad shape of the measured RL gain curve. This supports the mechanism suggested by the theory: checkpoints differ not only in their current reward, but also in the quality of their costate signal and their remaining reward headroom.

Figure~\ref{fig:costates_z_score} evaluates the final predictive story. The left panel compares the realized mean post-RL reward against $\widehat{R}_{\rm after}=R_{\rm before}+{\rm Affine}(P_{\rm gap})$, where the affine map calibrates the scale of $P_{\rm gap}$ to the observed reward-gain scale. The right panel uses a scale-free version, $z(R_{\rm before}) + z(P_{\rm gap})$, and compares it to $z(\overline{R}_{\rm after})$, where $z(\cdot)$ denotes normalization across checkpoints. This asks whether pretrained competence plus predicted RL readiness explains the final post-RL checkpoint quality. Empirically, the combined predictor correlates more strongly with averaged post-RL reward than either current reward or the impact score alone. Across the RL budgets, the score satisfies ${\rm Spearman}(P_{\rm gap},\overline{\Delta R})\approx 0.60$, while the combined scale-free predictor satisfies ${\rm Spearman}(z(R_{\rm before})+z(P_{\rm gap}),z(\overline{R}_{\rm after}))\approx 0.73$. These results suggest that the theory captures a real checkpoint-dependent RL-readiness signal, but also that current competence remains necessary for predicting the final post-RL reward level.

\begin{figure}[t!] 
    \centering
    \includegraphics[width=0.42\linewidth]{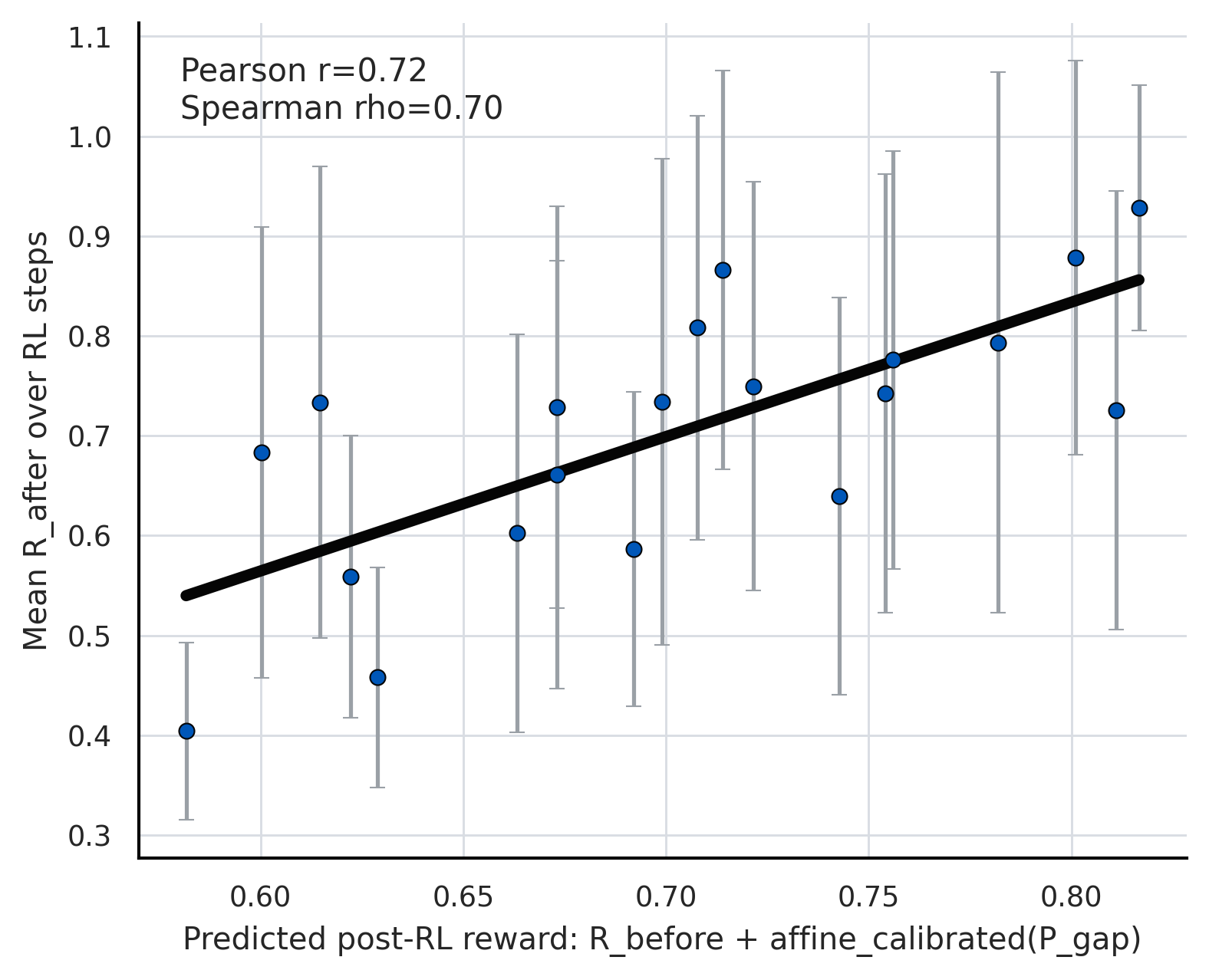}
    \vspace{0pt}
    \includegraphics[width=0.57\linewidth]{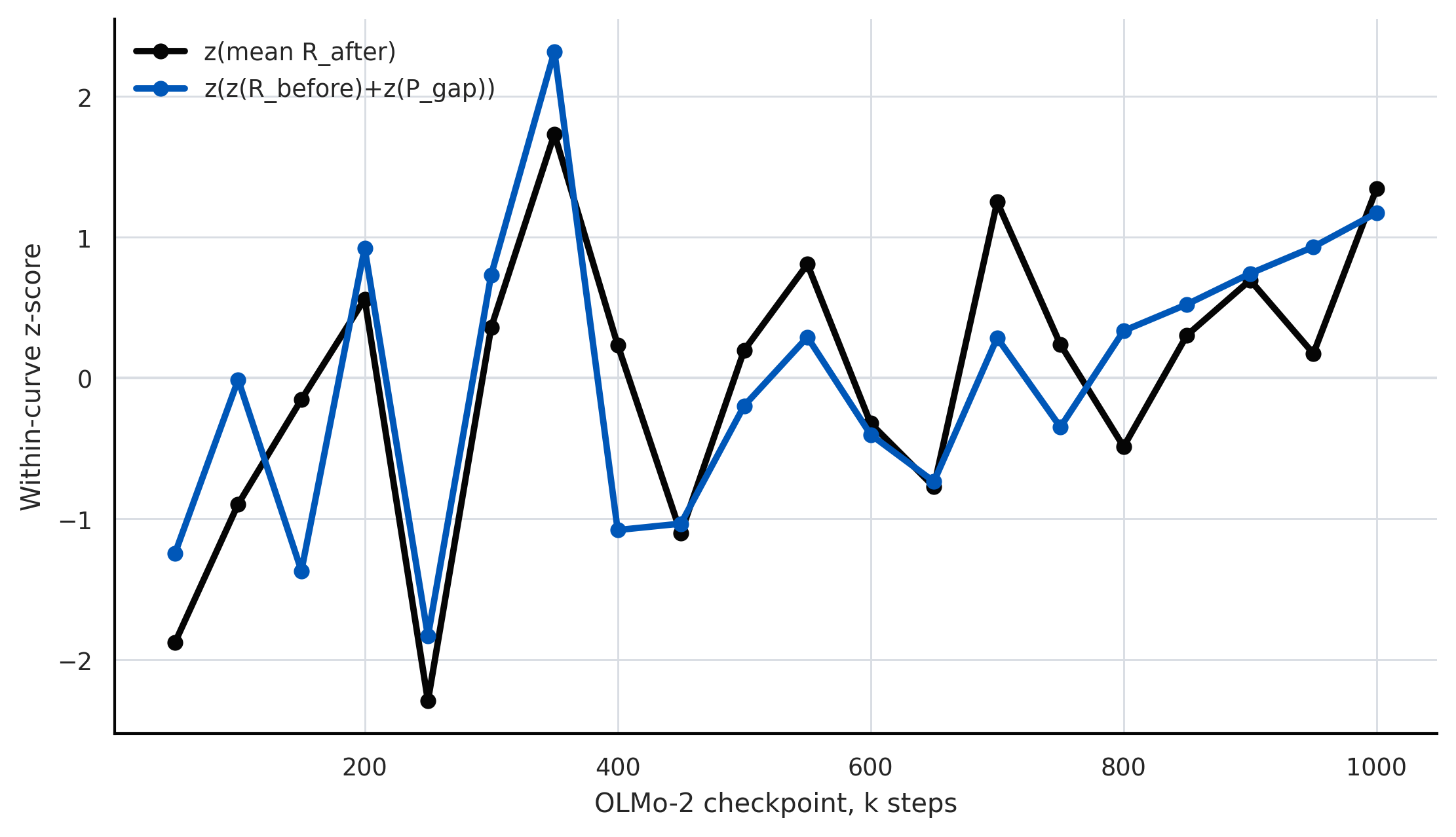}
    \vspace{-5mm}
    \caption{Z-scores of the gain after RL vs. z-scores of the predicted RL impact. Correlation (left) and curve per checkpoint (right). The std bars on left indicate RL gain variance over the various training steps.}
    \label{fig:costates_z_score}
    \vspace{-5mm}
\end{figure} 

\textbf{Limitations}. The current setup is controlled and partially toy-like. The reward is deliberately chosen to be differentiable in the model logits, which allows exact costate measurements. Many parameters affect the observed correspondence, including reward scale, RL learning rate, batch size, group size, and the number of RL updates. Our method may indicate that a checkpoint is ready for efficient RL, but how RL is performed remains an additional optimization question. Indeed, too few RL steps under-realize the predicted headroom, while too many steps saturate the reward and obscure checkpoint differences. For this reason, we report averages across several RL budgets. 

\vspace{-3mm}
\section{Related Work and Discussion}
\vspace{-3mm}
There is a large body of work on extending the GRPO algorithm. Methods such as these are essentially additional techniques built on top of GRPO, and our methodology extrapolates well with respect to such modifications.In terms of the critic-free RL interpretation, it has been shown that GRPO implicitly computes a value-like quantity via U-statistics \citep{zhou2026demystifying}. However, this account is specific to GRPO, while PPO \citep{schulman2017proximal} and even REINFORCE also perform unexpectedly well in the LLM setting. Moreover, evidence from standard RL suggests that group-normalized episodic baselines are generally insufficient to replace a critic for bootstrapping and temporal credit assignment \citep{deOliveira}. Therefore, we argue that the gains of RL in LLMs come not only from the GRPO baseline design.

\cite{zheng2025soft} propose using BPTT diretcly instead of RL and obtain better empirical results, which according to our interpretation should make the value-gradient signal propagates better. Another interesting observation is that scalar continuous rewards were used in the latest DeepSeek-V4 model \citep{deepseekai2026deepseekv4} to achieve better results. Recently, value gradient flow-based learning was proposed \citep{xu2026reinforcement}, which achieved sustainable results in RL tasks, signaling that the value signal is useful. However, this method differs from ours because their approach is actor-free, showing how the value gradient can be used as a policy, while we focus on the critic-free LLMs.
\vspace{-3mm}
\section{Conclusion}
\vspace{-3mm}
This paper argues that critic-free RL in LLM post-training works because the actor backward pass is not value-free: it carries a value-gradient-like signal. In a differentiable rollout this signal is exact in expectation, and in discrete transformers it survives approximately because attention provides a differentiable channel for temporal credit transport while the missing sampling path is controlled by policy entropy. This perspective leads to a simple prediction: RL gains should be largest when a checkpoint has both strong usable credit-assignment signal and remaining reward headroom. More broadly, the paper suggests that the success of critic-free RL in LLMs is not an exception to credit-assignment theory, but a consequence of hidden-state computation in continues space.


\bibliography{iclr2026_conference}

@article{shao2024deepseekmath,
  title={Deepseekmath: Pushing the limits of mathematical reasoning in open language models},
  author={Shao, Zhihong and Wang, Peiyi and Zhu, Qihao and Xu, Runxin and Song, Junxiao and Bi, Xiao and Zhang, Haowei and Zhang, Mingchuan and Li, YK and Wu, Yang and others},
  journal={arXiv preprint arXiv:2402.03300},
  year={2024}
}

@inproceedings{fairbank2012value,
  title={Value-gradient learning},
  author={Fairbank, Michael and Alonso, Eduardo},
  booktitle={The 2012 international joint conference on neural networks (ijcnn)},
  pages={1--8},
  year={2012},
  organization={IEEE}
}

@inproceedings{rezende2014stochastic,
  title={Stochastic backpropagation and approximate inference in deep generative models},
  author={Rezende, Danilo Jimenez and Mohamed, Shakir and Wierstra, Daan},
  booktitle={International conference on machine learning},
  pages={1278--1286},
  year={2014},
  organization={PMLR}
}

@article{schulman2015gradient,
  title={Gradient estimation using stochastic computation graphs},
  author={Schulman, John and Heess, Nicolas and Weber, Theophane and Abbeel, Pieter},
  journal={Advances in neural information processing systems},
  volume={28},
  year={2015}
}

@article{schulman2017proximal,
  title={Proximal policy optimization algorithms},
  author={Schulman, John and Wolski, Filip and Dhariwal, Prafulla and Radford, Alec and Klimov, Oleg},
  journal={arXiv preprint arXiv:1707.06347},
  year={2017}
}

@article{deOliveira,
  title={Learning Without Critics? Revisiting GRPO in Classical Reinforcement Learning Environments.},
  author={de Oliveira, Bryan LM, et al},
  journal={Arxiv},
  volume={12},
  year={2025}
}

@article{zhou2026demystifying,
  title={Demystifying Group Relative Policy Optimization: Its Policy Gradient is a U-Statistic},
  author={Zhou, Hongyi and Ye, Kai and Xu, Erhan and Zhu, Jin and Gong, Shijin and Shi, Chengchun},
  journal={arXiv preprint arXiv:2603.01162},
  year={2026}
}

@article{zheng2025soft,
  title={Soft-grpo: Surpassing discrete-token llm reinforcement learning via gumbel-reparameterized soft-thinking policy optimization},
  author={Zheng, Zhi and Gu, Yu and Liu, Wei and Teh, Yee Whye and Lee, Wee Sun},
  journal={arXiv preprint arXiv:2511.06411},
  year={2025}
}

@misc{olmo20242olmo2furious,
      title={{2 OLMo 2 Furious}},
      author={{Team OLMo} and Pete Walsh and Luca Soldaini and Dirk Groeneveld and Kyle Lo and Shane Arora and Akshita Bhagia and Yuling Gu and Shengyi Huang and Matt Jordan and Nathan Lambert and Dustin Schwenk and Oyvind Tafjord and Taira Anderson and David Atkinson and Faeze Brahman and Christopher Clark and Pradeep Dasigi and Nouha Dziri and Michal Guerquin and Hamish Ivison and Pang Wei Koh and Jiacheng Liu and Saumya Malik and William Merrill and Lester James V. Miranda and Jacob Morrison and Tyler Murray and Crystal Nam and Valentina Pyatkin and Aman Rangapur and Michael Schmitz and Sam Skjonsberg and David Wadden and Christopher Wilhelm and Michael Wilson and Luke Zettlemoyer and Ali Farhadi and Noah A. Smith and Hannaneh Hajishirzi},
      year={2024},
      eprint={2501.00656},
      archivePrefix={arXiv},
      primaryClass={cs.CL},
      url={https://arxiv.org/abs/2501.00656},
}

@misc{deepseekai2026deepseekv4,
      title={DeepSeek-V4: Towards Highly Efficient Million-Token Context Intelligence},
      author={DeepSeek-AI},
      year={2026},
}

@article{xu2026reinforcement,
  title={Reinforcement Learning via Value Gradient Flow},
  author={Xu, Haoran and Hu, Kaiwen and Sojoudi, Somayeh and Zhang, Amy},
  journal={arXiv preprint arXiv:2604.14265},
  year={2026}
}
\bibliographystyle{iclr2026_conference}

\newpage
\appendix


\section{Proofs}
\label{app:proofs}

This appendix collects the proofs and supporting derivations for the statements
used in the main text. We keep the notation of the main paper: the
reparameterized action is written as
\(a_t=\pi_\theta(s_t,\xi_t)\), the differentiable transition as
\(s_{t+1}=f_\theta(s_t,a_t)\), and the costate as
\(\lambda_t=\partial R_t/\partial s_t\).

\subsection{Pathwise differentiation under a differentiable rollout}

Under Definition~\ref{def:reparam}, all stochasticity enters through the
exogenous noise variables \(\xi_{1:T}\), whose law is independent of
\(\theta\). Hence, for a fixed noise realization, the sampled return
\(R_1(\theta,\xi_{1:T})\) is a deterministic differentiable function of
\(\theta\). Therefore,
\[
J(\theta)
=
\mathbb{E}_{\xi_{1:T}}\!\left[R_1(\theta,\xi_{1:T})\right]
=
\int R_1(\theta,\xi_{1:T})\,p(\xi_{1:T})\,d\xi_{1:T}.
\]
Under the differentiability and domination assumptions implicit in
Definition~\ref{def:reparam}, differentiation may be interchanged with
integration:
\[
\frac{\partial J(\theta)}{\partial\theta}
=
\int
\frac{\partial R_1(\theta,\xi_{1:T})}{\partial\theta}
p(\xi_{1:T})\,d\xi_{1:T}
=
\mathbb{E}\!\left[
\frac{\partial R_1}{\partial\theta}
\right].
\]
Thus, in the differentiable rollout model, the policy-gradient direction can
be computed pathwise by differentiating the unrolled trajectory computation.

\subsection{Proof of Proposition~\ref{prop:adjoint}: adjoint recursion}

\begin{proof}
Recall that
\[
R_t
=
r(s_t,a_t)+\gamma R_{t+1},
\qquad
a_t=\pi_\theta(s_t,\xi_t),
\qquad
s_{t+1}=f_\theta(s_t,a_t),
\qquad
\lambda_t:=\frac{\partial R_t}{\partial s_t}.
\]
Throughout the proof, the noise variables \(\xi_{t:T}\) are held fixed.

First differentiate \(R_t\) with respect to \(s_t\). The immediate reward
contributes the total derivative
\[
D r(s_t,a_t)
=
\frac{\partial r(s_t,a_t)}{\partial s}
+
\left(\frac{\partial \pi_\theta(s_t,\xi_t)}{\partial s}\right)^{\!\top}
\frac{\partial r(s_t,a_t)}{\partial a}.
\]
The future return contributes through \(s_{t+1}=f_\theta(s_t,a_t)\):
\[
\gamma
\left(
\frac{\partial f_\theta(s_t,a_t)}{\partial s}
+
\frac{\partial f_\theta(s_t,a_t)}{\partial a}
\frac{\partial \pi_\theta(s_t,\xi_t)}{\partial s}
\right)^{\!\top}
\lambda_{t+1}
=
\gamma
\bigl(D f_\theta(s_t,a_t)\bigr)^{\!\top}
\lambda_{t+1}.
\]
Therefore,
\[
\lambda_t
=
D r(s_t,a_t)
+
\gamma
\bigl(D f_\theta(s_t,a_t)\bigr)^{\!\top}
\lambda_{t+1},
\]
which is the adjoint recursion~\eqref{eq:adjoint}. It remains to derive the parameter-gradient expression. Let
\[
\dot s_t := \frac{d s_t}{d\theta}.
\]
The total state sensitivity satisfies
\[
\dot s_{t+1}
=
D f_\theta(s_t,a_t)\,\dot s_t
+
\frac{\partial f_\theta(s_t,a_t)}{\partial\theta}
+
\frac{\partial f_\theta(s_t,a_t)}{\partial a_t}
\frac{\partial \pi_\theta(s_t,\xi_t)}{\partial\theta}.
\]
The total derivative of the return is
\[
\frac{dR_1}{d\theta}
=
\sum_{t=1}^{T}\gamma^{t-1}
\left[
\dot s_t^{\top}D r(s_t,a_t)
+
\left(
\frac{\partial \pi_\theta(s_t,\xi_t)}{\partial\theta}
\right)^{\!\top}
\frac{\partial r(s_t,a_t)}{\partial a_t}
\right].
\]
Using the adjoint identity
\[
D r(s_t,a_t)
=
\lambda_t
-
\gamma
\bigl(D f_\theta(s_t,a_t)\bigr)^{\!\top}
\lambda_{t+1},
\]
we obtain
\[
\begin{aligned}
\frac{dR_1}{d\theta}
=
\sum_{t=1}^{T}\gamma^{t-1}
\Bigg[
&
\left(
\frac{\partial \pi_\theta(s_t,\xi_t)}{\partial\theta}
\right)^{\!\top}
\frac{\partial r(s_t,a_t)}{\partial a_t}
+
\dot s_t^{\top}\lambda_t
\\
&-
\gamma
\dot s_t^{\top}
\bigl(D f_\theta(s_t,a_t)\bigr)^{\!\top}
\lambda_{t+1}
\Bigg].
\end{aligned}
\]
From the state-sensitivity recursion,
\[
D f_\theta(s_t,a_t)\,\dot s_t
=
\dot s_{t+1}
-
\frac{\partial f_\theta(s_t,a_t)}{\partial\theta}
-
\frac{\partial f_\theta(s_t,a_t)}{\partial a_t}
\frac{\partial \pi_\theta(s_t,\xi_t)}{\partial\theta}.
\]
Substituting this into the previous display gives
\[
\begin{aligned}
\frac{dR_1}{d\theta}
=
\sum_{t=1}^{T}\gamma^{t-1}
\Bigg[
&
\left(
\frac{\partial \pi_\theta(s_t,\xi_t)}{\partial\theta}
\right)^{\!\top}
\frac{\partial r(s_t,a_t)}{\partial a_t}
\\
&+
\gamma
\left(
\frac{\partial f_\theta(s_t,a_t)}{\partial\theta}
+
\frac{\partial f_\theta(s_t,a_t)}{\partial a_t}
\frac{\partial \pi_\theta(s_t,\xi_t)}{\partial\theta}
\right)^{\!\top}
\lambda_{t+1}
\Bigg]
\\
&+
\sum_{t=1}^{T}\gamma^{t-1}
\left[
\dot s_t^{\top}\lambda_t
-
\gamma
\dot s_{t+1}^{\top}\lambda_{t+1}
\right].
\end{aligned}
\]
The second sum telescopes:
\[
\sum_{t=1}^{T}\gamma^{t-1}
\left[
\dot s_t^{\top}\lambda_t
-
\gamma
\dot s_{t+1}^{\top}\lambda_{t+1}
\right]
=
\dot s_1^{\top}\lambda_1
-
\gamma^T\dot s_{T+1}^{\top}\lambda_{T+1}.
\]
The initial state is the prompt state and is independent of \(\theta\), so
\(\dot s_1=0\). Also \(\lambda_{T+1}=0\) by definition. Hence the telescoping
term vanishes. Collecting the remaining terms yields
\[
\begin{aligned}
\frac{dR_1}{d\theta}
=
\sum_{t=1}^{T}\gamma^{t-1}
\Bigg[
&
\gamma
\left(
\frac{\partial f_\theta(s_t,a_t)}{\partial\theta}
\right)^{\!\top}
\lambda_{t+1}
\\
&+
\left(
\frac{\partial \pi_\theta(s_t,\xi_t)}{\partial\theta}
\right)^{\!\top}
\left(
\frac{\partial r(s_t,a_t)}{\partial a_t}
+
\gamma
\left(
\frac{\partial f_\theta(s_t,a_t)}{\partial a_t}
\right)^{\!\top}
\lambda_{t+1}
\right)
\Bigg].
\end{aligned}
\]
Taking expectation over the exogenous noise gives~\eqref{eq:bptt-grad}.
\end{proof}

\subsection{Proof of identity~\eqref{eq:costate-is-vg}: costates are value-gradient estimators}

\begin{proof}
By definition,
\[
V_t^\pi(s)
=
\mathbb{E}\!\left[R_t \mid s_t=s\right].
\]
Conditioned on \(s_t=s\), the remaining randomness comes only from the future
exogenous noises \(\xi_{t:T}\), whose distribution is independent of \(s\) and
\(\theta\). Therefore, under the regularity conditions of
Definition~\ref{def:reparam},
\[
\begin{aligned}
G_t^\pi(s)
&=
\frac{\partial V_t^\pi(s)}{\partial s}
=
\frac{\partial}{\partial s}
\mathbb{E}\!\left[R_t\mid s_t=s\right]
\\
&=
\mathbb{E}\!\left[
\frac{\partial R_t}{\partial s_t}
\;\middle|\;
s_t=s
\right]
=
\mathbb{E}\!\left[\lambda_t\mid s_t=s\right].
\end{aligned}
\]
Thus the BPTT costate is a Monte Carlo sample whose conditional expectation is
the value gradient.
\end{proof}

\subsection{Proof of Lemma~\ref{lem:shift}: SF \(=\) PD under a shift policy}

\begin{proof}
Fix \(s\). Let
\[
\pi_\theta(a\mid s)=\nu(a-\bar a_\theta(s)),
\qquad
u:=a-\bar a_\theta(s).
\]
Then \(a=u+\bar a_\theta(s)\), and the density of \(u\) is \(\nu(u)\),
independent of \(\theta\). By the chain rule,
\[
\frac{\partial}{\partial\theta}\log \pi_\theta(a\mid s)
=
-
\left(
\frac{\partial \bar a_\theta(s)}{\partial\theta}
\right)^{\!\top}
\frac{\partial}{\partial u}\log \nu(u).
\]
Therefore,
\[
\begin{aligned}
\mathbb{E}\!\left[
r(s,a)
\frac{\partial}{\partial\theta}
\log \pi_\theta(a\mid s)
\right]
&=
-
\left(
\frac{\partial \bar a_\theta(s)}{\partial\theta}
\right)^{\!\top}
\int
r(s,u+\bar a_\theta(s))
\frac{\partial}{\partial u}\log\nu(u)
\nu(u)\,du
\\
&=
-
\left(
\frac{\partial \bar a_\theta(s)}{\partial\theta}
\right)^{\!\top}
\int
r(s,u+\bar a_\theta(s))
\frac{\partial \nu(u)}{\partial u}
\,du.
\end{aligned}
\]
Integrating by parts and using the assumed vanishing boundary terms,
\[
\int
r(s,u+\bar a_\theta(s))
\frac{\partial \nu(u)}{\partial u}
\,du
=
-
\int
\nu(u)
\frac{\partial r(s,u+\bar a_\theta(s))}{\partial u}
\,du.
\]
Since
\[
\frac{\partial r(s,u+\bar a_\theta(s))}{\partial u}
=
\frac{\partial r(s,a)}{\partial a},
\]
we obtain
\[
\mathbb{E}\!\left[
r(s,a)
\frac{\partial}{\partial\theta}
\log \pi_\theta(a\mid s)
\right]
=
\left(
\frac{\partial \bar a_\theta(s)}{\partial\theta}
\right)^{\!\top}
\mathbb{E}\!\left[
\frac{\partial r(s,a)}{\partial a}
\right].
\]
This proves~\eqref{eq:sf-pd-equiv}.
\end{proof}

\subsection{Local GRPO/PPO value-gradient form}
\label{app:local-grpo-vg}

The main text uses the following consequence of Lemma~\ref{lem:shift} and
Proposition~\ref{prop:adjoint}.

\begin{corollary}[GRPO is a value-gradient update in expectation]
\label{cor:grpo_is_vg}
Assume the differentiable rollout of Definition~\ref{def:reparam}, the
shift/additive-noise policy of Lemma~\ref{lem:shift}, and treat the GRPO/PPO
advantage weights \(\widehat A_{i,t}\) as stop-gradient scalars. Locally at
\(\theta=\theta_{\mathrm{old}}\), where clipping is inactive, the expected
actor-gradient direction of the GRPO/PPO surrogate is pathwise-equivalent to
a BPTT adjoint update. The backward signal in that adjoint update satisfies
\[
\mathbb{E}[\lambda_t\mid s_t=s]
=
G_t^\pi(s).
\]
Thus the local GRPO/PPO actor update is value-gradient-like in expectation.
\end{corollary}

\begin{proof}
At \(\theta=\theta_{\mathrm{old}}\),
\[
\rho_{i,t}(\theta_{\mathrm{old}})
=
\frac{
\pi_{\theta_{\mathrm{old}}}(o_t^i\mid s_{i,t})
}{
\pi_{\theta_{\mathrm{old}}}(o_t^i\mid s_{i,t})
}
=
1.
\]
In a local neighbourhood in which the clipping threshold is not active, the
clipped surrogate reduces to the unclipped likelihood-ratio term. Hence,
with \(\widehat A_{i,t}\) held fixed,
\[
\frac{\partial}{\partial\theta}
\left(
\rho_{i,t}(\theta)\widehat A_{i,t}
\right)
\bigg|_{\theta=\theta_{\mathrm{old}}}
=
\widehat A_{i,t}
\frac{\partial}{\partial\theta}
\log \pi_\theta(o_t^i\mid s_{i,t})
\bigg|_{\theta=\theta_{\mathrm{old}}}.
\]
Therefore, the local actor part of the GRPO/PPO gradient is the usual
score-function policy-gradient term:
\[
\mathbb{E}\!\left[
\frac{1}{G}
\sum_{i=1}^{G}
\frac{1}{T_i}
\sum_{t=1}^{T_i}
\widehat A_{i,t}
\frac{\partial}{\partial\theta}
\log \pi_\theta(o_t^i\mid s_{i,t})
\right]_{\theta=\theta_{\mathrm{old}}}.
\]
The KL penalty in~\eqref{eq:grpo_objective} contributes an ordinary
differentiable regularization gradient and does not affect the costate
identity.

Now apply Lemma~\ref{lem:shift} token-by-token. For any differentiable local
scalar signal \(h_{i,t}(s_t,a_t)\) whose stop-gradient coefficient is
\(\widehat A_{i,t}\), we have
\[
\mathbb{E}\!\left[
\widehat A_{i,t}h_{i,t}(s_t,a_t)
\frac{\partial}{\partial\theta}
\log \pi_\theta(a_t\mid s_t)
\right]
=
\left(
\frac{\partial \bar a_\theta(s_t)}{\partial\theta}
\right)^{\!\top}
\mathbb{E}\!\left[
\frac{\partial\bigl(\widehat A_{i,t}h_{i,t}(s_t,a_t)\bigr)}
{\partial a_t}
\right].
\]
Taking \(h_{i,t}(s_t,a_t)=r(s_t,a_t)\) gives exactly
\eqref{eq:weighted_sf_to_pd}. Thus the score-function form and the pathwise
action-derivative form have the same expectation under the shift policy.

Summing these pathwise terms over the rollout gives the BPTT gradient of the
differentiable trajectory. By Proposition~\ref{prop:adjoint}, the backward
state-sensitivity propagated by this BPTT computation is the costate
\(\lambda_t\). By identity~\eqref{eq:costate-is-vg},
\[
\mathbb{E}[\lambda_t\mid s_t=s]
=
G_t^\pi(s).
\]
Hence the local GRPO/PPO actor update is value-gradient-like in expectation.
\end{proof}

\subsection{Empirical costate recursion in a transformer}

\begin{proof}[Derivation of~\eqref{eq:costate-recursion}]
For a fixed trajectory, the local GRPO loss without clipping can be written as
\[
L(\theta)
=
\sum_{k=1}^{T}
\widehat A_k \ell_k(\theta),
\qquad
\ell_k
=
\log \pi_\theta(o_k\mid s_k).
\]
By Definition~\ref{def:empirical-costate},
\[
\hat\lambda_t
=
\frac{\partial}{\partial h_t^{(L)}}
\sum_{k\ge t}
\widehat A_k \ell_k.
\]
Splitting the \(k=t\) term from the future terms gives
\[
\hat\lambda_t
=
\widehat A_t
\frac{\partial \ell_t}{\partial h_t^{(L)}}
+
\sum_{k>t}
\widehat A_k
\frac{\partial \ell_k}{\partial h_t^{(L)}}.
\]
For \(k>t\), the dependence of \(\ell_k\) on \(h_t^{(L)}\) passes through the
transformer computation graph. In particular, the first future step contributes
through the Jacobian
\[
J_{t+1\leftarrow t}
=
\frac{\partial h_{t+1}^{(L)}}{\partial h_t^{(L)}}.
\]
Applying the chain rule recursively gives
\[
\sum_{k>t}
\widehat A_k
\frac{\partial \ell_k}{\partial h_t^{(L)}}
=
J_{t+1\leftarrow t}^{\top}
\frac{\partial}{\partial h_{t+1}^{(L)}}
\sum_{k\ge t+1}
\widehat A_k \ell_k
=
J_{t+1\leftarrow t}^{\top}
\hat\lambda_{t+1}.
\]
Therefore,
\[
\hat\lambda_t
=
\widehat A_t
\frac{\partial \ell_t}{\partial h_t^{(L)}}
+
J_{t+1\leftarrow t}^{\top}
\hat\lambda_{t+1},
\qquad
\hat\lambda_{T+1}=0,
\]
which is~\eqref{eq:costate-recursion}.
\end{proof}

\subsection{Proof of Proposition~\ref{prop:sampling-gap}: sampling-gap bound}

\begin{proof}
The exact relaxed transition Jacobian decomposes as in
\eqref{eq:jacobian-decomp}:
\[
D f_\theta^{\mathrm{exact}}
=
J_{t+1\leftarrow t}^{\attn}
+
\frac{\partial h_{t+1}^{(L)}}{\partial e_t}
\frac{\partial e_t}{\partial o_t}
\frac{\partial o_t}{\partial z_t}
\frac{\partial z_t}{\partial h_t^{(L)}}.
\]
Thus,
\[
\left\|
D f_\theta^{\mathrm{exact}}
-
J_{t+1\leftarrow t}^{\attn}
\right\|
\le
\left\|
\frac{\partial h_{t+1}^{(L)}}{\partial e_t}
\right\|
\left\|
\frac{\partial e_t}{\partial o_t}
\right\|
\left\|
\frac{\partial o_t}{\partial z_t}
\right\|
\left\|
\frac{\partial z_t}{\partial h_t^{(L)}}
\right\|.
\]
The first, second, and fourth factors are controlled by the embedding matrix,
the output head, and the local Lipschitz constants of the transformer block.
It remains to control the effective Jacobian through sampling.

For a soft or straight-through relaxation, the categorical sample is replaced
by a differentiable soft sample with probabilities
\[
p_t=\softmax(z_t/\tau),
\]
where \(\tau>0\) is the relaxation temperature. Its Jacobian has the form
\[
\frac{\partial o_t}{\partial z_t}
=
\frac{1}{\tau}
\left(
\mathrm{diag}(p_t)-p_t p_t^{\top}
\right).
\]
The matrix
\(\mathrm{diag}(p_t)-p_t p_t^{\top}\) is the covariance matrix of a categorical
one-hot vector. Since it is positive semidefinite,
\[
\left\|
\mathrm{diag}(p_t)-p_t p_t^{\top}
\right\|
\le
\mathrm{Tr}
\left(
\mathrm{diag}(p_t)-p_t p_t^{\top}
\right)
=
1-\|p_t\|_2^2.
\]
Let \(p_{\max}:=\max_j p_{t,j}\). Then
\[
1-\|p_t\|_2^2
\le
1-p_{\max}^2
\le
2(1-p_{\max}).
\]
For a fixed vocabulary size \(|V|\), normalized entropy controls the
distance from a point mass: there exists a finite constant \(C_V\) such that
\[
1-p_{\max}
\le
C_V
\sqrt{\frac{H_t}{\log |V|}}.
\]
This follows because the simplex is compact and the ratio
\[
\frac{1-p_{\max}}{\sqrt{H_t/\log |V|}}
\]
has a finite continuous extension at the deterministic vertices, where both
the numerator and the entropy vanish. Combining the previous bounds gives
\[
\left\|
\frac{\partial o_t}{\partial z_t}
\right\|
\le
\frac{2C_V}{\tau}
\sqrt{\frac{H_t}{\log |V|}}.
\]
Absorbing the fixed vocabulary, temperature, embedding, head, and local
Lipschitz constants into a single constant \(C\), we obtain
\[
\bigl\|
D f_\theta^{\mathrm{exact}}
-
J_{t+1\leftarrow t}^{\attn}
\bigr\|
\le
C
\sqrt{\frac{H_t}{\log |V|}},
\]
which is~\eqref{eq:gap-bound}.
\end{proof}

\subsection{Proof of Proposition~\ref{prop:attn-rank}: attention-pathway rank and magnitude}

\begin{proof}
Consider a single causal-attention head. Its output at position \(t\) has the
form
\[
\mathrm{Attn}_t
=
\sum_{t'\le t}
\alpha_{t,t'} W_V h_{t'}.
\]
For a fixed earlier position \(t'\), the Jacobian of this head output with
respect to \(h_{t'}\) maps into the \(d_{\mathrm{head}}\)-dimensional value
subspace of that head. Hence its rank is at most \(d_{\mathrm{head}}\). This
remains true when including the derivative of the attention weights
\(\alpha_{t,t'}\), because the output of a single head is still a vector in
\(\mathbb{R}^{d_{\mathrm{head}}}\).

With \(n_{\mathrm{heads}}\) heads, the multi-head attention output is the
concatenation and output projection of \(n_{\mathrm{heads}}\) such head
outputs. Therefore, the rank of the cross-position attention contribution in
one layer is at most
\[
d_{\mathrm{head}}\,n_{\mathrm{heads}}.
\]
Across \(L\) layers, the attention-pathway Jacobian can be written as a sum of
cross-position contributions transported through the intervening layer
Jacobians. Since the rank of a sum is at most the sum of ranks, the total
attention-pathway rank is bounded by
\[
L\,d_{\mathrm{head}}\,n_{\mathrm{heads}}.
\]
This proves the stated rank bound. For the magnitude claim, each head contribution contains factors of the
attention weight \(\alpha_{t,t'}\), value projection \(W_V\), output projection,
and softmax-score derivatives. Under bounded projection norms and bounded
hidden-state norms, the contribution from position \(t'\) is therefore
controlled by the size of the corresponding attention weights. Thus broadly
distributed attention provides more cross-position credit pathways, while
small attention weights suppress the corresponding Jacobian entries.
\end{proof}

\subsection{Proof of Theorem~\ref{thm:main}: discrete GRPO costates approximate BPTT costates}

\begin{proof}
Define the conditional costate error
\[
\delta_t
:=
\mathbb{E}[\hat\lambda_t\mid h_t]
-
\mathbb{E}[\lambda_t\mid h_t].
\]
The empirical costate recursion is
\[
\hat\lambda_t
=
\widehat A_t
\frac{\partial \ell_t}{\partial h_t^{(L)}}
+
\left(J_{t+1\leftarrow t}^{\attn}\right)^{\!\top}
\hat\lambda_{t+1}.
\]
The exact relaxed BPTT recursion has the form
\[
\lambda_t
=
D r_t
+
\gamma
\left(D f_\theta^{\mathrm{exact}}\right)^{\!\top}
\lambda_{t+1}.
\]
Under the local SF-to-PD identification of Lemma~\ref{lem:shift}, the
immediate credit terms align in expectation. Thus the difference between the
two conditional recursions is caused by the propagated future error and by
the missing sampling-path Jacobian:
\[
\delta_t
=
\gamma
\left(J_{t+1\leftarrow t}^{\attn}\right)^{\!\top}
\delta_{t+1}
+
\gamma
\left(
D f_\theta^{\mathrm{exact}}
-
J_{t+1\leftarrow t}^{\attn}
\right)^{\!\top}
\mathbb{E}[\lambda_{t+1}\mid h_t].
\]
Taking norms gives
\[
\|\delta_t\|
\le
\gamma
\left\|
J_{t+1\leftarrow t}^{\attn}
\right\|
\|\delta_{t+1}\|
+
\gamma
\left\|
D f_\theta^{\mathrm{exact}}
-
J_{t+1\leftarrow t}^{\attn}
\right\|
\left\|
\mathbb{E}[\lambda_{t+1}\mid h_t]
\right\|.
\]
Let
\[
\epsilon_{t+1}
:=
\left\|
D f_\theta^{\mathrm{exact}}
-
J_{t+1\leftarrow t}^{\attn}
\right\|
\left\|
\mathbb{E}[\lambda_{t+1}\mid h_t]
\right\|.
\]
By Proposition~\ref{prop:sampling-gap}, this per-step error is controlled by
the entropy-dependent sampling-gap bound. Hence
\[
\|\delta_t\|
\le
\gamma
\left\|
J_{t+1\leftarrow t}^{\attn}
\right\|
\|\delta_{t+1}\|
+
\gamma \epsilon_{t+1}.
\]
Unrolling this inequality backward from \(\delta_{T+1}=0\) yields
\[
\|\delta_t\|
\le
\sum_{k=t+1}^{T}
\gamma^{k-t}
\left(
\prod_{j=t+1}^{k-1}
\left\|
J_{j\leftarrow j-1}^{\attn}
\right\|
\right)
\epsilon_k.
\]
The main text writes the Jacobian chain compactly as
\(\|J_{k\leftarrow k-1}^{\attn}\|^{k-t-1}\), which should be read as this
product or as a uniform upper bound on the product. Therefore,
\[
\bigl\|
\mathbb{E}[\hat\lambda_t\mid h_t]
-
\mathbb{E}[\lambda_t\mid h_t]
\bigr\|
\le
\sum_{k=t+1}^{T}
\gamma^{k-t}
\cdot
\underbrace{
\bigl\|J_{k\leftarrow k-1}^{\attn}\bigr\|^{k-t-1}
}_{\text{Jacobian chain growth}}
\cdot
\epsilon_k,
\]
which is~\eqref{eq:main-bound}.
\end{proof}

\subsection{Proof of the usable-signal lower bound}

\begin{proof}
Let
\[
u_t^{(m)}
:=
\mathbb{E}[\hat\lambda_t^{(m)}\mid h_t],
\qquad
v_t
:=
G_t.
\]
Then
\[
\langle u_t^{(m)},v_t\rangle
=
\|v_t\|_2^2
+
\langle u_t^{(m)}-v_t,v_t\rangle.
\]
By Cauchy--Schwarz,
\[
\langle u_t^{(m)}-v_t,v_t\rangle
\ge
-
\|u_t^{(m)}-v_t\|_2\,\|v_t\|_2.
\]
Therefore,
\[
\langle u_t^{(m)},v_t\rangle
\ge
\|G_t\|_2^2
-
\|G_t\|_2
\left\|
\mathbb{E}[\hat\lambda_t^{(m)}\mid h_t]-G_t
\right\|_2.
\]
Taking expectation over \(q,\tau,t\) gives
\[
\mathbb E_{q,\tau,t}
\left[
\left\langle
\mathbb E[\hat\lambda_t^{(m)}\mid h_t],
G_t
\right\rangle
\right]
\ge
\Sigma(\theta)-\Lambda_m(\theta),
\]
with \(\Sigma(\theta)\) and \(\Lambda_m(\theta)\) defined in
\eqref{eq:sigma-lambda}. This proves~\eqref{eq:signal-lower-bound}.
\end{proof}

\end{document}